\documentclass{l4dc2026}

\title[Encoding Dynamics in IRL Reward Shaping]{Enhancing Inverse Reinforcement Learning through
Encoding Dynamic Information in Reward Shaping}
\usepackage[T1]{fontenc}
\usepackage{times}

\usepackage{amsmath,amsfonts,bm}

\def\eqref#1{equation~\ref{#1}}

\def\1{\bm{1}}

\DeclareMathAlphabet{\mathsfit}{\encodingdefault}{\sfdefault}{m}{sl}
\SetMathAlphabet{\mathsfit}{bold}{\encodingdefault}{\sfdefault}{bx}{n}

\usepackage{microtype}
\usepackage{amsfonts}
\usepackage{booktabs} 
\usepackage{times}
\usepackage{algorithm}
\usepackage{algorithmic}
\usepackage{makecell} 
\usepackage{multirow}  
\usepackage{array} 
\usepackage{siunitx}

\usepackage{graphicx} 
\usepackage{subcaption}
\usepackage{float}
\usepackage{wrapfig}
\usepackage{xcolor}         
\usepackage[dvipsnames]{xcolor} 

\usepackage{hyperref}
\makeatletter
\newcommand{\cref}[1]{\autoref{#1}}      
\newcommand{\Cref}[1]{\Autoref{#1}}      

\makeatletter
\@ifundefined{assumption}{\newtheorem{assumption}{Assumption}}{}
\@ifundefined{remark}{}{}
\makeatother

\newcommand\boldblue[1]{\textcolor{blue}{\textbf{#1}}}
\newcommand{\tref}[1]{Theorem~\ref{#1}}
\newcommand{\aref}[1]{Assumption~\ref{#1}}

\author{
 \Name{Simon Sinong Zhan$^{*1}$} \Email{SinongZhan2028@u.northwestern.edu}\\
 \Name{Philip Wang$^{*1}$} \Email{philipwang2025@u.northwestern.edu}\\
 \Name{Qingyuan Wu$^{*2}$} \Email{qingyuan.wu@soton.ac.uk}\\ 
 \Name{Ruochen Jiao$^1$} \Email{RuochenJiao2024@u.northwestern.edu}\\
 \Name{Yixuan Wang$^1$} \Email{yixuanwang2024@u.northwestern.edu}\\
 \Name{Chao Huang$^2$} \Email{Chao.Huang@soton.ac.uk}\\ 
 \Name{Qi Zhu$^1$} \Email{qzhu@northwestern.edu}\\
 \addr $^1$  Department of Electrical and Computer Engineering, Northwestern University, USA\\
 \addr $^2$ School of Electronics and Computer Science, University of Southampton, UK\\
}

\begin{document}
\def\thefootnote{*}\footnotetext{These authors contributed equally to this work}
\maketitle


\begin{abstract}%
Adversarial-based inverse reinforcement learning (IRL) has shown promising results using reward shaping under deterministic settings, but it struggles in stochastic environments.
To address this issue, we propose a novel maximum causal entropy based off-policy IRL method with transition-aware reward shaping framework. 
Our method integrates transition model estimation directly to learn stochastic-invariant rewards. 
We conduct a thorough theoretical analysis, establishing bounds on reward error and performance differences to validate the effectiveness of our method.  
The experimental results in continuous locomotion tasks (MuJoCo) show that our method can achieve superior performance in stochastic environments and competitive performance in deterministic environments, with significant improvement in sample efficiency, compared to existing baselines. 
Additionally, we extend our framework to high-dimensional vision-based tasks, where our method shows promising results on multiple stochastic Atari games. 
These results demonstrate that embedding transition awareness into reward learning is critical for robust IRL in realistic stochastic settings.
\end{abstract}


\section{Introduction}
Reinforcement learning (RL) has achieved considerable success across various domains, including board game~\citep{schrittwieser2020mastering}, MOBA game~\citep{berner2019dota}, time-delayed system~\citep{wuboosting,wu2024variational}, and cyber-physical systems~\citep{wang2023empowering,wang2023joint,wang2023enforcing,zhan2024state}. Despite these advances, RL highly depends on the quality of reward function design which demands expertise, intensive labour, and a great amount of time~\citep{russell1998learning}. To address this, imitation learning (IL) methods, such as Behavior Cloning (BC)~\citep{torabi2018behavioral} and Inverse Reinforcement Learning (IRL)~\citep{arora2021survey}, leverage human or expert demonstrations to bypass the need for explicit reward functions. These methods aim to learn from the demonstrations to eventually match the distribution of expert behavior, and have shown great promise in applications like autonomous driving \citep{codevilla2018end,sun2018fast}, legged locomotion \citep{peng2020learning,ratliff2009learning}, and planning tasks \citep{choudhury2018data,yin2022planning}.

The notable approaches within IRL are Generative Imitation Learning~(GIL) methods that build upon maximum entropy framework~\citep{ziebart2008maximum}. 
These methods frame imitation learning as a maximum likelihood estimation problem on trajectory distributions, converting it into a Boltzmann distribution parameterized by rewards under \textbf{deterministic} environment settings~\citep{wu2024variational}. 
This closely mirrors the distribution approximation found in generative models~\citep{finn2016connection,swamy2021moments}. 
Thus, model-free deep IL approaches often follow generative model structures, such as GANs \citep{ho2016generative,fu2017learning} or diffusion models \citep{reuss2023goal,wu2024diffusing}, and require extensive sampling for distribution/score function matching in on-policy fashion \citep{orsini2021matters}. 
Model-based IL frameworks have also emerged, where model-based framework is designed to provide estimation for gradient and planning, leading to innovative combinations such as gradient-based IRL with model predictive control (MPC) \citep{das2021model} and end-to-end differentiable IRL frameworks for complex robotics tasks \citep{baram2016model, baram2017end, sun2021adversarial, rafailov2021visual}.  
However, these approaches primarily address deterministic settings and struggle when applied to stochastic environments.

Existing IRL methods, rooted in their maximum entropy nature, mostly exclusively focus on learning "deterministic" reward techniques. 
These methods, however, face significant performance degeneration in stochastic environments, leading to \textbf{risk-seeking behavior} and increased data requirements \citep{ziebart2010modeling}. 
For example, an agent trained under the deterministic Markov Decision Process (MDP) might aim to imitate expert behavior by seeking high rewards, yet fail to account for the low probability of some transitions in stochastic MDP settings. 
This happens because, in stochastic environments, the assumption of maximum entropy no longer holds. 
The trajectory distributions are not aligned with a Boltzmann distribution solely parameterized by \textbf{deterministic} rewards. 
In this case, the dynamic information must also be incorporated into the formulation. 
There are two likely solutions: One is massive sampling to cover all possible outcomes, which is computationally expensive in large state action spaces~\citep{devlin2011theoretical,gupta2022unpacking}; 
The other is changing from maximum entropy framework to maximum causal entropy framework, estimating the dynamics information, and integrating it into the reward design, making the reward "stochastic". 
Traditional reward design is usually based on state only $R(s_t)$~\citep{torabi2018generative}, state-action pair $R(s_t, a_{t})$~\citep{blonde2019sample}, or transition tuple $R(s_t, a_t, s_{t+1})$~\citep{fu2017learning}, where the information inputted can be thought as a \textbf{deterministic} sample piece under the \textbf{stochastic} setting. 
The challenge in \textbf{stochastic} environments calls for a different perspective of rewards -- stochastic rewards absorbing the transition information. 
More detailed literature survey has been has been investigated in Appendix \ref{section::related_work}.

Inspired by this idea, we propose a novel maximum causal entropy off-policy model-based adversarial IRL framework with a specifically tailored transition-aware reward shaping approach to elevate performance in stochastic environments while remaining competitive in deterministic settings. 
In contrast to existing methods, our approach leverages the predictive power of the estimated transition model to shape rewards, represented as $\hat{R}(s_t, a_t, \hat{\mathcal{T}})$. 
This also enables us to guide policy optimization and reduce dependency on costly real-world interactions in a model-based fashion. 
As part of our analysis, we provide a theoretical guarantee on the optimal behavior for policies induced by our reward shaping and derive a bound on the performance difference with respect to the transition model learning errors. 
Empirically, we demonstrate that this integration significantly enhances sample efficiency and policy performance in both settings, providing a comprehensive solution to the limitations of existing IRL methods in uncertain environments.
\textbf{Contributions of this work} include: 
\vspace{-0.25cm}
\begin{itemize}
    \item A transition-aware reward shaping formulation $\hat{R}(s_t, a_t, \hat{\mathcal{T}})$ with model estimation $\hat{\mathcal{T}}$ for stochastic MDPs, which provides the optimal policy invariance guarantee. 
    \vspace{-0.25cm}
    \item A novel model-based off-policy IRL framework rooted in \textbf{maximum causal entropy} theory that seamlessly incorporates transition model training, adversarial reward learning with model estimation and forward model-based RL process, enhancing performance in stochastic environments, and sample efficiency.
    \vspace{-0.25cm}
    \item Theoretical analysis on reward learning and performance difference under transition model learning errors within our framework.
    \vspace{-0.25cm}
    \item Empirical validation showing our approach's performance improvements in stochastic environments as well as significant gains in sample efficiency and comparable performance in deterministic environments.
\end{itemize}

We begin by introducing the necessary preliminaries on Markov Decision Processes (MDPs) and inverse reinforcement learning (IRL). We then present our model-enhanced reward shaping method, along with its theoretical guarantee. Next, we describe the full model-enhanced IRL framework, including its derivation from the maximum causal entropy objective, and provide theoretical analysis on the reward error bound and performance difference bound. Experimental results are reported on a range of MuJoCo and Atari benchmarks, demonstrating the effectiveness of our approach. We conclude with a summary of our findings.

\vspace{-0.3cm}

\section{Preliminaries}
\label{section::preliminaries}
\paragraph{MDP.}
RL usually considers a Markov Decision Process (MDP) $\mathcal{M}$~\citep{puterman2014markov} denoted as a tuple \(\langle\mathcal{S},\mathcal{A},\mathcal{T},\gamma, R, \rho_0\rangle\). $\rho_0$ is the initial distribution of the state. $s\in\mathcal{S},a\in\mathcal{A}$ stands for the state and action space respectively. $\mathcal{T}$ stands for the transition dynamic such that $\mathcal{T}:\mathcal{S}\times\mathcal{A}\times\mathcal{S}\rightarrow[0,1]$. $\gamma\in(0,1)$ is the discounted factor, $R$ stands for reward function such that $R:\mathcal{S}\times\mathcal{A}\rightarrow\mathbb{R}$ and $\|R\|_{\infty}\leq R_{\max}$. The undiscounted visitation distribution of trajectory $\tau$ with policy $\pi$ is given by $p(\tau) = \rho_0\prod_{t=0}^{T-1}\mathcal{T}(s_{t+1}|s_t,a_t)\pi(a_t|s_t).$
The objective function of RL is defined as $\max\mathbb{E}_{\tau\sim p(\tau)}\left[\sum_{t=0}^{T}\gamma^t R(\tau)-H(\pi)\right]$, where $H$ is the log likelihood of the policy. We introduce Soft Value Iteration for bellmen update~\citep{haarnoja2018soft}, where $Q^{soft}$ and $V^{soft}$ denotes the soft Q function and Value function respectively:
\[
\begin{aligned}
    \label{equation::soft_VI}
    V^{soft}(s_t) = \log\sum_{a_t\in\mathcal{A}}\exp{Q^{soft}(s_t, a_t)}, &\quad
    Q^{soft}(s_t,a_t) = R(s_t, a_t) + \gamma\mathbb{E}_{\mathcal{T}}\left[V^{soft}(s_{t+1})|s_t,a_t\right],\\
    \pi(a_t|s_t)=\exp{(Q^{soft}(s_t,a_t) - V^{soft}(s_t))}, &\quad A^{soft}(s_t,a_t)=Q^{soft}(s_t,a_t)-V^{soft}(s_t) 
\end{aligned}
\]

\paragraph{Inverse RL.} In the IRL setting, we usually consider the MDP without reward as $\mathcal{M}'$ where $R$ is also unknown. We denote the data buffer $\mathcal{D}_{exp}$ which collects trajectories from an expert policy $\pi^E$. We consider a reward function $R_\theta:\mathcal{S}\times\mathcal{A}\rightarrow\mathbb{R}$, where $\theta$ is the reward parameter. An IRL problem can be defined as a pair $\mathcal{B}=(\mathcal{M}',\pi^E)$. A reward function $R_\theta$ is feasible for $\mathcal{B}$ if $\pi^E$ is an optimal policy for the MDP $\mathcal{M}'\cup R_\theta$, and we denote the set of feasible rewards as $\mathcal{R}_{\mathcal{B}}$. 
Using the maximize likelihood estimation framework, we can formulate the IRL as the following maximum causal entropy (MCE) problem, $\arg\max_{\theta} \mathbb{E}_{\tau\sim\mathcal{D}_{exp}}\log p_\theta(\tau)$, where $Q_{R_\theta}^{soft}$ and $V^{soft}_{R_\theta}$ are based on $R_\theta$ and $p_\theta(\tau)\propto\rho_0\prod_{t=0}^{T-1}\mathcal{T}(s_{t+1}|s_t,a_t)\exp(Q^{soft}_{R_\theta}(s_t,a_t)-V^{soft}_{R_\theta}(s_t))$~\citep{ziebart2010modeling}. Under \textbf{deterministic} MDP, the above problem can be simplified as maximum entropy (ME) problem, where $p_\theta(\tau)\propto\frac{1}{Z_\theta}\exp{\sum_{t=0}^{T-1} R_\theta(s_t,a_t)}$ and $Z_\theta$ is the temperature factor of the Boltzmann Distribution~\citep{ziebart2008maximum}.

\section{Transition-Aware Reward Shaping}
\label{section::reward_shaping}

\begin{table*}[t]
\caption{This table summarizes different reward formulations and their dynamic properties. \textbf{Components} refer to the input pairs for the reward functions. \textbf{Reward Shaping} indicates whether additional physical potential information is included (\textbf{X} means none). \textbf{Dynamics} specifies if transitions are considered in the reward function.}
\centering
\scalebox{0.8}{
\begin{tabular}{c|ccc}
\hline
Methods        
    & Components
    & Reward Shaping
    & Dynamic Information
\\ \hline
AIRL(State Only)~\citep{fu2017learning}
    & $s_t$
    & $R(s_t)+\textrm{constant}$
    & X
\\
DAC~\citep{kostrikov2018discriminator}
    & $s_t, a_t$
    & X
    & X
\\
SAM~\citep{blonde2019sample}
    & $s_t,a_t$
    & X
    & X
\\
SQIL~\citep{reddy2019sqil}
    & $s_t,a_t$
    & binary 
    & X
\\
GAIfO~\citep{torabi2018generative}
    & $s_t, s_{t+1}$
    & X
    & single sample
\\
Ours
    & $s_t, a_t, \mathcal{T}$
    & $R(s_t,a_t)+\gamma\mathbb{E}_{\mathcal{T}}[\phi(s_{t+1})|s_t,a_t]-\phi(s_t)$
    & transition model
    
\\
\hline
\end{tabular}
}
\label{reward_shaping_table}
\end{table*}
In this section, we illustrate the advantages of involving transition dynamics into the reward shaping, especially in stochastic MDP settings. 
Existing work has proposed various formulations and definitions~(\cref{reward_shaping_table}), but few considers transition dynamic information in the reward shaping. 
Defining rewards solely based on states, $R^{s}(s_t)$, offers limited utility in environments where actions are critical. 
Even though the state-action pair-based rewards $R^{sa}(s_t, a_t)$ can capture the missing information of the action taken, it fails to consider future information, the successive state $s_{t+1}$. 
Transition tuple-based rewards $R^{tuple}(s_t, a_t, s_{t+1})$ incorporate dynamic information in a sampling-based way, which requires abundant data to learn the underlying relationship of two consecutive states, potentially raising the sample efficiency issue in the stochastic environment with the huge state space. 
To address this issue, we propose transition-aware reward shaping $\hat{R}(s_t,a_t,\mathcal{T})$, which explicitly infuses dynamic information $\mathcal{T}$ on the potential function, thus significantly improving sample efficiency. 
Specifically, our reward shaping is defined as 
\begin{equation}
\label{def::transition_reward}
   \hat{R}(s_t,a_t,\mathcal{T})=R(s_t,a_t)+\gamma\mathbb{E}_{\mathcal{T}}\left[\phi(s_{t+1})\vert s_t,a_t\right]-\phi(s_t), 
\end{equation}
where $\phi$ is a state-only potential function, $\mathcal{T}$ is the dynamics. Another insight of the above reward shaping is to resemble the advantage function with the soft value function as the optimal potential function, which we will elaborate on next section. With the given reward shaping $\hat{R}$, it is crucial to show that it induces the same optimal behavior as the original reward $R$. We formally define this policy invariance property as follows. 
\begin{definition}\citep{memarian2021self}
\label{def::soft_optimal_policy}
    Let $R$ and $\hat{R}$ be two reward functions. We say they induce the same soft optimal policy under transition dynamics $\mathcal{T}$ if, for all states $s\in\mathcal{S}$ and actions $a\in\mathcal{A}$:$A_{R}^{soft}(s_t,a_t) = A_{\hat{R}}^{soft}(s_t,a_t).$
\end{definition}
With the above definition, we can transfer the proof of policy invariant property of our designed reward shaping (\cref{def::transition_reward}) to showing the equivalence of soft advantage functions, which is proved in the following theorem. The detailed proof can be found in Appendix~\ref{proof::soft_optimal_policy}.
\begin{theorem}[Policy Invariance]
\label{thm::soft_optimal_policy}
    Let $R$ and $\hat{R}$ be two reward functions. $R$ and $\hat{R}$ induce the same soft optimal policy under all transition dynamics $\mathcal{T}$ if $\hat{R}(s_t,a_t,\mathcal{T})=R(s_t,a_t)+\gamma\mathbb{E}_{\mathcal{T}}[\phi(s_{t+1})\vert s_t,a_t] - \phi(s_t)$ for some potential-shaping function $\phi:\mathcal{S}\rightarrow\mathbb{R}$.
\end{theorem}

\tref{thm::soft_optimal_policy} implies that the optimal policy induced from our transition-aware reward shaping $\hat{R}$ (\cref{def::transition_reward}) is equivalent to the optimal policy trained by the ground-truth reward function $R$ under the soft value iteration fashion.

\section{Model Enhanced IRL}
\label{section::algorithm}
\begin{figure*}[!t]
    \centering
    \includegraphics[width=0.9\linewidth]{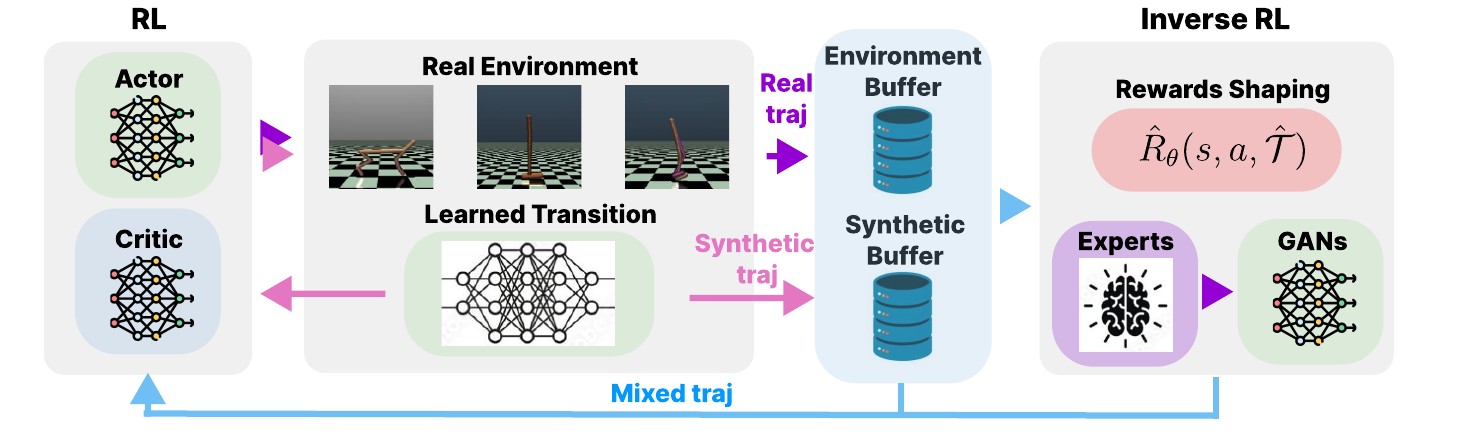}
    \caption{Framework overview of Model-Enhanced Adversarial IRL. Different color arrows stand for different sample flows. \textcolor{Purple}{Purple} stands for real environmental interaction samples, \textcolor{Rhodamine}{pink} stands for synthetic samples generated from learned transition model, and \textcolor{Cyan}{blue} stands for mixed of both.}
    \label{fig::framework}
\end{figure*}
In this section, we first elaborate on the adversarial training of our reward shaping (\cref{equation::discriminator}) and present the theoretical insight (\cref{prop::gradient_align}) of the equivalence between cross-entropy training loss of adversarial reward shaping formulation and maximum log-likelihood loss of original maximum causal entropy IRL problem. 
Then, we showcase our practical algorithm framework with trajectory generation, transition model learning, and policy optimization in the loop, as shown in ~\cref{fig::framework}. 
Furthermore, we theoretically investigate the reward function bound (\tref{thm:reward_func_Error_Bound}) and performance difference bound (\tref{thm:value_error_bound}) under the transition model learning error.

\subsection{Adversarial Formulation of Reward Shaping}
\label{section::adversarial_reward}
In this section, we connect the reward shaping in the adversarial training framework with rewards learning objective under the MCE framework. Inspired by GANs~\citep{goodfellow2014generative}, the idea behind the adversarial framework is to train a binary discriminator $D(s_t,a_t,s_{t+1})$ or $D(s_t, a_t)$ to distinguish state-action-transition samples from an expert and those generated by imitator policy following the original ME setting. However, as mentioned above, we only take in state-action pair and transition function to define our reward function which also extends to our discriminator as:
\begin{equation}
\label{equation::discriminator}
    D_\theta(s_t,a_t,\mathcal{T})=\frac{\exp\{f_{\theta}(s_t,a_t,\mathcal{T})\}}{\exp\{f_{\theta}(s_t,a_t,\mathcal{T})\}+\pi(a_t\vert s_t)},
\end{equation}
where $f_\theta(s_t,a_t,\mathcal{T})=R_\theta(s_t,a_t)+\gamma\mathbb{E}_{\mathcal{T}}\left[\phi_\theta(s_{t+1})|s_t,a_t\right]-\phi_\theta(s_t)$ resembles the reward shaping defined above. The loss function for the training discriminator is defined below.
\begin{equation}
\label{equation::disc_loss}
    \mathcal{L}_{disc}=-\mathbb{E}_{\mathcal{D}_{exp}}\left[\log D_{\theta}(s,a,\mathcal{T})\right]-\mathbb{E}_{\pi}\left[\log(1-D_\theta(s,a,\mathcal{T}))\right].
\end{equation}
We bridge this adversarial formulation with the original MCE IRL problem. 
In the following proposition, we give a sketch of the proof to show the connection between the objective function of the discriminator and MCE IRL. 
Proof details can be found in Appendix \ref{proof::gradient alignment}.
\begin{proposition}
\label{prop::gradient_align}
    Consider an undiscounted MDP. Suppose $f_\theta$ and $\pi$ at the current iteration are the soft-optimal advantage function and policy for reward function $R_\theta$. Minimising the cross-entropy loss of the discriminator under generator $\pi$ is equivalent to maximising the log-likelihood under Maximum Causal Entropy IRL. 
\end{proposition}
With the above proposition, we can construct an intuition that $f^*_\theta$ should be equal to $\hat{R}_\theta$ the reward shaping we introduced early and resemble the soft advantage function. To extract rewards to represent reward used for policy optimization, we use $\log(D_\theta(s,a,\mathcal{T}))-\log(1-D_\theta(s,a,\mathcal{T}))$, which resembles the entropy-regularized reward shaping $f_\theta(s,a,\mathcal{T})-\log\pi(a|s)$. 

\subsection{Algorithm Framework}
\label{section::algorithm_framework}
In this section, we present the overall framework of Model-Enhanced IRL and illustrate how transition model training is incorporated into the learning loop. We use ensemble gaussian dynamic models to adopt the transition~\citep{yu2020mopo,yu2021combo,rigter2022rambo}. The transition model is updated in each policy optimization iteration similar as model-based RL approaches~\citep{janner2019trust,hansen2022temporal,zhan2024state}. At each iteration, the updated transition model is utilized for reward learning and synthetic data generation in eval mode, which is stored in the synthetic trajectory replay buffer. Unlike AIRL and GAIL, our framework operates in an \textbf{off-policy} fashion, where samples used for both discriminator and policy update are drawn from a combination of the environmental replay buffer and the synthetic replay buffer. An overview of our framework is shown in ~\cref{fig::framework}, and detailed algorithmic steps and parameters are provided in Algorithm \ref{alg:meairl}.  

\paragraph{Sample Efficiency:}
\label{section:sample_efficiency} To improve sample efficiency, we leverage the estimated transition model to generate $H$-steps synthetic trajectories data alongside real interaction data, facilitating policy optimization. Given that the estimated transition model is inaccurate at the beginning, we employ a dynamic ratio between real and synthetic data to prevent the model from being misled by unlikely synthetic transitions~\citep{janner2019trust,zhan2024state}. Specifically, early-stage generated trajectories are not stored persistently, unlike real interactions which are fully stored in the off-policy environmental replay buffer. To maintain training stability, we use a synthetic replay buffer with a size that gradually increases as training progresses, ensuring a balanced inclusion of synthetic data over time. The growth rates of the data ratio and buffer size are adjusted based on the complexity of the transition model learning process and can be fine-tuned via hyper-parameters. Details can be found in Appendix \ref{sec::implementation}.

\paragraph{Distribution Shift:} To mitigate distribution shift~\citep{lee2020addressing, lin2020model} during training, we employ a strategy involving the learned transition model. Typically, during interaction, the real state $s_t$ is used as input to the actor, and the resulting action $a_t$ is applied in the environment. To incorporate the transition model, we predict a synthetic state $\hat{s}_t$ from previous $s_{t-1}$ and $a_{t-1}$. This generated $\hat{s}_t$ is then fed into the actor to produce action $\hat{a}_t$. The actions $a_t$ and $\hat{a}_t$ are mixed and applied to the environment with a certain ratio, and the resulting pairs $(s_t, a_t)$ or $(s_t, \hat{a}_t)$ are stored in the environmental replay buffer. This approach helps balance the exploration of real and model-predicted dynamics, reducing the impact of distributional discrepancies. 

\subsection{Performance Analysis}
\label{section::performance_gap}
In this section, we analyze the optimal performance bound in the presence of transition model learning errors. Our results show that as the transition model error approaches zero, the performance difference at the optimal point vanishes at the same time. The learned transition model $\hat{\mathcal{T}}$ persists in some errors compared with the ground-true transition dynamic. In this section, we investigate how this error will affect performance of our method. extending original representation, we define an IRL problem as $\mathfrak{B}=(\mathcal{M}',\pi^E)$, where $\mathcal{M}'$ is a MDP without $R$ and $\pi^E$ is an optimal expert policy. We denote $\mathcal{R}_{\mathfrak{B}}$ as the set of feasible rewards set for $\mathfrak{B}$. Since under our case $\mathcal{T}$ is approximated by $\hat{\mathcal{T}}$, we have another IRL problem defined as $\hat{\mathfrak{B}}=(\hat{\mathcal{M}'},\pi^E)$ where $\hat{\mathcal{M}}'$ has the same state and action space, discount factor, and initial distribution but an estimated transition model $\hat{\mathcal{T}}$. For notation, we use \( D_{\text{TV}} \) to denote the total variation distance, \( \Vert \cdot \Vert \) to represent the infinity norm (with \( \infty \) omitted for simplicity), \( \vert \mathcal{S} \vert \) to denote the cardinality of the state space, and \( V_{\mathcal{M}' \cup R}^{\pi^*} \) to represent the value function of policy \( \pi^* \) under the MDP \( \mathcal{M}' \) with reward \( R \), and vice versa.
\begin{assumption}[Transition Model Error]
\label{assumption:transition_error}
    Since the transition model is trained through a supervised fashion, we can use a PAC generalization bound~\citep{shalev2014understanding} for sample error. Therefore, we assume that the total variation distance between $\mathcal{T}$ and $\hat{\mathcal{T}}$ is bounded by $\epsilon_\mathcal{T}$ through $[0,T]$: $\max_t\mathbb{E}_{s\sim\pi_{D,t}}\left[D_{TV}(\mathcal{T}(s'|s,a)\vert\hat{\mathcal{T}}(s'|s,a))\right]\leq\epsilon_\mathcal{T}$, which is a common assumption adopted in literature~\citep{janner2019trust,sikchi2022learning}.  
\end{assumption} 
Next, based on the assumed total visitation bound on transition models (\aref{assumption:transition_error}), we aim to propagate this bound to the reward learning process through our model-infused reward shaping. Specifically, the error bound arises from using the approximated transition model \(\hat{\mathcal{T}}\) in our reward shaping, instead of the true but inaccessible transition model \(\mathcal{T}\). 
\begin{theorem}[Reward Function Error Bound]
\label{thm:reward_func_Error_Bound}
    Let $\mathfrak{B}=(\mathcal{M}',\pi^*)$ and $\hat{\mathfrak{B}}=(\hat{\mathcal{M}'},\pi^*)$ be two IRL problems with transition functions $\mathcal{T}$ and $\hat{\mathcal{T}}$ respectively, then for any $R^E\in\mathcal{R}_{\mathfrak{B}}$ there is a corresponding $\hat{R}^E\in\mathcal{R}_{\hat{\mathfrak{B}}}$ such that 
    \begin{equation*}
        \Vert R^E-\hat{R}^E\Vert\leq\frac{\gamma}{1-\gamma}|\mathcal{S}|\epsilon_{\mathcal{T}}R_{\max}
    \end{equation*}
\end{theorem}
Proof of ~\tref{thm:reward_func_Error_Bound} can be found in \cref{proof:reward_error_bound}. With rewards bound above, we can extend the bound to the final performance, which represents the value functions difference brought up by estimated transition model error under RL setting.   
\begin{theorem}[Performance Difference Bound]
\label{thm:value_error_bound}
    The performance difference between the optimal policies ($\pi^*$ and $\hat{\pi}^*$) in corresponding MDPs ($\mathcal{M}'\cup R$ and $\hat{\mathcal{M}}'\cup \hat{R}$) can be bounded as follows:
    \begin{equation*}
        \Vert V_{\mathcal{M}'\cup R^E}^{\pi^*} - V_{\hat{\mathcal{M}}'\cup \hat{R}^E}^{\hat{\pi}^*} \Vert\leq
        \epsilon_\mathcal{T}\left[ \frac{\gamma}{(1-\gamma)^2} R_{\max}+\frac{1+\gamma}{(1-\gamma)^2}R_{\max} |\mathcal{S}|\right].\\
    \end{equation*}
\end{theorem}
The detailed proof of ~\tref{thm:value_error_bound} is presented in \cref{proof:value_error_bound}. The above theorem highlights the relationship between the performance difference and the transition model error, also implying that a perfectly-learned transition model ($\epsilon_\mathcal{T} \rightarrow 0$) could make the performance difference negligible. We also extend an ablation study to empirically validate this performance bound in ablation studies in \cref{section:ablation_model}, where we can observe a \textit{linear-kind} tendency illustrating in above theorem.

\section{Experiments}
\label{section::experiment}
We evaluate our Model-Enhanced IRL algorithm on the MuJoCo and Atari\texttt{100k} benchmarks~\citep{mujoco, bellemare13arcade, kaiser2019model} with an NVIDIA RTX 4090 GPU and Intel Core-i9 13900K CPU. Our algorithm is tested against other Adversarial Imitation Learning (AIL) methods, including on-policy algorithms GAIL \citep{ho2016generative}, AIRL \citep{fu2017learning}, and the off-policy method Discriminator Actor-Critic (DAC) \citep{kostrikov2018discriminator} across multiple MuJoCo environments under both stochastic and deterministic settings. We use Proximal Policy Optimization (PPO) \citep{schulman2017proximal} for both GAIL and AIRL, and Soft Actor-Critic (SAC) \citep{haarnoja2018soft} for DAC as the forward RL optimization respectively. In addition, we apply our transition-aware reward shaping framework to the high-dimensional vision-based Atari\texttt{100k} task and compare it with CNN-AIRL \citep{tucker2018inverse}. Complete learning curves of all methods are provided in Supplementary Material. The experiments are designed to highlight the key advantages of our framework:
\begin{itemize}
    \vspace{-0.25cm}
    \item \textbf{Performance in Stochastic Environments:} In stochastic settings, our method significantly outperforms other approaches. This enhanced ability to learn under uncertainty is attributed to our framework’s effectiveness in leveraging model-based predictiton capability.
    \vspace{-0.25cm}
    \item \textbf{Sample Efficiency} Our method can consistently reach expert performance with fewer training steps even under limited expert demonstrations when most of the baselines fail to extract reward signals. 
    \vspace{-0.25cm}
    \item \textbf{Performance in Deterministic Environments:} We demonstrate that our method is on par with existing baselines' performances in deterministic settings.
    \vspace{-0.15cm}
\end{itemize}

\paragraph{MuJuCo Experiment Setup.}
We choose 5 different MuJoCo continuous control tasks listed in \cref{performance_table_stoch}. To simulate stochastic dynamics, we introduce the agent-unknown random Gaussian noise with a mean of 0 and a standard deviation of 0.5 to the environmental interaction steps. All the expert trajectories are collected by an expert agent trained with standard SAC~\citep{haarnoja2018soft} under deterministic or stochastic MuJoCo environments.
Each algorithm is trained with 100k environmental steps and evaluated every 1k steps for \texttt{InvertedPendulum-v4} and \texttt{InvertedDoublePendulum-v4}. For \texttt{Hopper-v4}, \texttt{HalfCheetah-v4}, and \texttt{Walker2d-v4}, AIRL and GAIL are trained with 10M steps with evaluations every 100k steps, whereas DAC and our algorithm are trained with 1M environmental steps with evaluations every 10k steps. Each experiment is repeated with five random seeds, and we vary the number of expert demonstrations from 5 to 1000 to demonstrate robustness. 
\paragraph{Stochastic MuJoCo.}
\label{stoch_performance}
\begin{table*}[t]
\centering
\caption{Best performance of expert and all algorithms in \textbf{stochastic} MuJoCo Environments under conditions of different numbers of expert trajectories provided (10, 100, and 1000). AIRL and GAIL are trained with \texttt{10M} environmental steps. DAC and Ours are trained with \texttt{1M} environmental steps.}
\scalebox{0.72}{
\begin{tabular}{c|c|ccccc}
\hline
Environment 
& \# Expert Trajactories 
& Expert 
& GAIL 
& AIRL
& DAC                   
& Ours                
\\ \hline
InvertedPendulum-v4        & 10       
& $1000.0_{\pm 0.0}$     
& $1000.0_{\pm 0.0}$     
& $1000.0_{\pm 0.0}$    
& $1000.0_{\pm 0.0}$    
& $1000.0_{\pm 0.0}$ \\
                           & 100      
& $1000.0_{\pm 0.0}$     
& $1000.0_{\pm 0.0}$     
& $1000.0_{\pm 0.0}$    
& $1000.0_{\pm 0.0}$    
& $1000.0_{\pm 0.0}$ \\
                           & 1000     
& $986.09_{\pm 95.97}$   
& $1000.0_{\pm 0.0}$     
& $1000.0_{\pm 0.0}$    
& $1000.0_{\pm 0.0}$    
& $1000.0_{\pm 0.0}$ \\ \hline
InvertedDoublePendulum-v4  & 10       
& $131.3_{\pm 77.0}$   
& $155.0_{\pm 58.1 }$  
& $163.0_{\pm 48.6 }$   
& $100.6_{\pm 11.8 }$
& \boldblue{$193.4_{\pm 15.5 }$} \\
                           & 100      
& $108.0_{\pm 43.2}$   
& $167.2_{\pm 26.6 }$    
& $151.2_{\pm 28.6 }$   
& $94.5_{\pm 9.9 }$    
& \boldblue{$198.1_{\pm 76.3 }$} \\
                           & 1000     
& $140.44_{\pm 76.62}$   
& \boldblue{$189.5_{\pm 28.8 }$}   
& $150.2_{\pm 18.3 }$   
& $105.6_{\pm 20.4 }$   
& \boldblue{$182.2_{\pm 29.6 }$} \\ \hline
Hopper-v4                  & 10       
& $1786.0_{\pm 803.0}$ 
& $1266.9_{\pm 366.2 }$   
& $2092.2_{\pm 57.4 }$  
& $1000.4_{\pm 5.3 }$
& \boldblue{$2408.4_{\pm 641.7 }$} \\
                           & 100      
& $1489.6_{\pm 659.6}$ 
& $2385.9_{\pm 350.0 }$  
& $2789.9_{\pm 30.8 }$  
& $993.1_{\pm 10.5 }$  
& \boldblue{$2820.9_{\pm 89.8 }$} \\
                           & 1000     
& $1516.0_{\pm 692.6}$ 
& $2746.5_{\pm 270.9 }$   
& $2744.3_{\pm 37.4 }$  
& $2007.1_{\pm 719.7 }$  
& \boldblue{$2858.8_{\pm 76.9}$} \\ \hline
HalfCheetah-v4             & 10       
& $1567.4_{\pm 74.1}$
& $368.5_{\pm 53.7 }$  
& $463.9_{\pm 61.2 }$
& $9.5_{\pm 457.2 }$ 
& \boldblue{$888.6_{\pm 67.3 }$} \\
                           & 100      
& $1120.5_{\pm 67.5}$
& $398.1_{\pm 123.5 }$  
& $556.0_{\pm 12.8 }$   
& $615.9_{\pm 250.5 }$
& \boldblue{$1108.3_{\pm 13.9 }$} \\
                           & 1000     
& $1113.5_{\pm 76.1}$  
& $735.6_{\pm 44.0 }$   
& $708.7_{\pm 14.5 }$ 
& $1046.4_{\pm 13.9 }$
& \boldblue{$1162.8_{\pm 62.2 }$} \\ \hline
Walker2d-v4                & 10       
& $3109.4_{\pm 1031.5}$ 
& $1262.8_{\pm 396.3 }$  
& $1170.5_{\pm 484.0 }$
& $101.4_{\pm 149.1 }$
& \boldblue{$2509.0_{\pm 860.0 }$} \\
                           & 100      
& $3295.4_{\pm 704.0}$ 
& $956.4_{\pm 313.2 }$  
& $1740.7_{\pm 609.8 }$
& $416.1_{\pm 243.2 }$ 
& \boldblue{$3311.0_{\pm 157.2 }$} \\
                           & 1000     
& $3268.9_{\pm 746.1}$ 
& $1430.6_{\pm 489.8 }$
& $3051.3_{\pm 210.5 }$
& \boldblue{$3531.3_{\pm 105.3 }$}
& \boldblue{$3497.8_{\pm 51.7 }$}
\\ \hline
\end{tabular}
}
\label{performance_table_stoch}
\end{table*}

In ~\cref{performance_table_stoch}, we present the performance of our method and baselines in stochastic MuJoCo environments with varying numbers of expert trajectories. \textbf{Our method consistently achieves the best performance across the majority of these environments, outperforming all baselines under different levels of expert trajectory availability}. In simple environments, such as \texttt{InvertedPendulum-v4}, the introduction of stochasticity and variations in expert trajectory have minimal impact on the final performance for both our method and the baselines. However, for more complex environments, the effect of stochasticity becomes more pronounced. Specifically, in \texttt{InvertedDoublePendulum-v4}, stochasticity notably degrades performance. Our method, however, maintains a competitive edge over all baselines, achieving better results with limited expert trajectories (10 and 100) and reaching similar performance to the baselines when more expert trajectories are available. In other environments with higher dimensional state space, our method substantially outperforms all baselines, especially when fewer expert trajectories are provided. These results indicate that our approach can effectively recover the reward function more closely from demonstrations in stochastic environments, resulting in significant performance improvement. To demonstrate our model's robustness to different levels of randomness and stochasticity, we further conduct an ablation study that can be found in Supplementary Material. Notably, the performance of DAC decreases significantly under stochastic settings, potentially due to DAC's ineffective reward formulation on state-action pairs, which also result in training instability shown in learning curves section in supplement. 

\paragraph{Sample Efficiency.}
\label{sample_efficiency}
In \cref{appendix:learning_curves}, we display the sample efficiency across various environments with different numbers of expert trajectories. Based on results, \textbf{our method shows significant superiority in sample efficiency across all of the environments under stochastic settings. This advantage becomes more apparent with limited expert trajectories.} Specifically, for \texttt{Hopper-v4}, \texttt{HalfCheetah-v4}, and \texttt{Walker2d-v4} our method is the only approach capable of reaching expert-level performance consistently with limited number of expert trajectories. GAIL and AIRL both fail to reach the expert within \texttt{1M} environmental training steps. DAC was able to reach expert performance only when expert trajectories are sufficient, though it still suffers from sample inefficiency and training instability. 
For \texttt{InvertedPendulum-v4}, all methods can achieve expert-level performance except DAC, which exhibits instability with limited demonstrations. In \texttt{InvertedDoublePendulum-v4},  introducing stochasticity into the dynamics makes it challenging for all algorithms to achieve reasonable performance from noisy expert demonstrations while DAC completely fails to reach expert-level performance. 
We also observe a universal trend across all stochastic environments: as the number of expert trajectories increases, both the sample efficiency and performance of all methods improve accordingly. In order to determine the specific contributions of our model-based policy optimization and model-infused reward shaping, we investigate their individual impacts on sample efficiency as detailed in ablation study in Supplementary Material. The results indicate that both components contribute significantly to the overall performance, with the model-based policy optimization playing a crucial role in enhancing sample efficiency, especially in environments with complex dynamics and high-dimensional state spaces.

\begin{table*}[!t]
\centering
\caption{Best performance of expert and all algorithms in \textbf{deterministic} MuJoCo Environments with 1000 expert trajectories provided. DAC and our methods are trained for \texttt{1M} environmental steps. GAIL and AIRL are trained for \texttt{10M} environmental steps.}
\scalebox{0.85}{
\begin{tabular}{c|ccccc}
\hline
Environment
    & Expert
    & GAIL
    & AIRL
    & DAC
    & Ours          
\\ \hline
InvertedPendulum-v4
    & $1000.0_{\pm 0.0}$
    & $1000.0_{\pm 0.0}$
    & $1000.0_{\pm 0.0}$
    & $1000.0_{\pm 0.0}$
    & $1000.0_{\pm 0.0}$
\\
InvertedDoublePendulum-v4
    & $9356.7_{\pm 0.2}$
    & $9324.4_{\pm 0.4 }$
    & $355.3_{\pm 76.3 }$
    & \boldblue{$9359.8_{\pm 0.1 }$}
    & \boldblue{$9359.8_{\pm 0.1 }$}
    \\
Walker2d-v4
    & $4520.7_{\pm 648.44}$
    & $3387.0_{\pm 617.8 }$
    & $3623.6_{\pm 189.6 }$
    & \boldblue{$4655.3_{\pm 126.4 }$}
    & $4396.7_{\pm 147.4 }$
\\
Hopper-v4
    & $3262.8_{\pm 314.4}$
    & $3420.8_{\pm 77.9 }$
    & $3385.8_{\pm 50.5 }$
    & $3481.6_{\pm 94.6 }$
    & \boldblue{$3506.6_{\pm 23.5 }$}
\\
HalfCheetah-v4
    & $13498.6_{\pm 710.9}$
    & $3502.6_{\pm 202.3}$
    & $ 3237.8_{\pm 85.5}$
    & \boldblue{$ 10102.2_{\pm 297.6}$}
    & $ 6509.8_{\pm 177.7}$
\\
\hline
\end{tabular}
}
\label{performance_table_det}
\end{table*}

\paragraph{Deterministic MuJoCo.}
The performance of deterministic MuJoCo environments can be found in ~\cref{performance_table_det}. 
For most tasks with deterministic dynamics, our method can achieve similar performance as the baselines and the expert. For \texttt{HalfCheetah-v4}, our method exceeds AIRL and GAIL, but fail to reach the similar level as DAC and expert. As the dynamic becomes complicated, ensemble dynamic model with our fixed hyper-parameters might not be able to fully capture the transition info leading to high compounding transition model error, which can cause the performance deficit. Our theoretical analysis supports this finding, and we will explore the efficacy of different dynamic model structures or more effective hyper-parameters setting for future works. Generally, \textbf{our method shows competitive performance with the baselines in the deterministic environments.}

\paragraph{Stochastic Atari Task.}

\begin{wraptable}{r}{5cm}
\caption{Evaluation return after \texttt{400k} steps on visual Atari.}
\resizebox{5cm}{!}{
    \begin{tabular}{@{}lcc@{}}
    \toprule
    Method & SpaceInvaders & BattleZone \\
    \midrule
    Expert      & $445.5_{\pm 129.0}$    & $22300.0_{\pm 3662.0}$ \\
    CNN-AIRL    & $148.0_{\pm 73.2}$     & $1200.0_{\pm 748.3}$ \\
    Ours        & \textbf{$323.0_{\pm 181.5}$} & $14400.0_{\pm 5607.1}$ \\
    \bottomrule
    \end{tabular}
}
\label{tab:atari_table}
\end{wraptable} 

In addition to the continuous MuJoCo locomotion suite, we further evaluate our reward shaping framework on two high-dimensional visual observation based Atari 2600 task, built on the Arcade Learning Enivonrment (ALE) simulator \citep{bellemare13arcade}. 
The environment is naturally stochastic with the built-in sticky actions. We employ the standard Atari\texttt{100k} evaluation protocol where the algorithms are trained for 400k steps, 100k interactions considering action repeat \citep{kaiser2019model}. We obtain the expert trajectories by training Dreamer-v2 \citep{hafner2020mastering} for \texttt{400k} steps and sample 10 trajectories afterwards. In order to capture the high-dimensional observation information, we substitute original ensemble model with the Recurrent State-Space Model (RSSM) to encode the dynamics. \cref{tab:atari_table} shows our method comparing to CNN-AIRL \citep{tucker2018inverse}, an adapted IRL framework that preprocesses the image input through CNN then use it for recovering rewards. From the result, our method can recover expert behavior to some extent under the high-dimensional stochastic environment. More details of this experiment can be found in Supplementary Materials.

\vspace{-0.5cm}
\section{Conclusion}
\label{section::conclusion}
In this paper, we presented a novel off-policy model-enhanced IRL framework starting from Maximum Causal Entropy theory by introducing transition-aware reward shaping, specifically designed to enhance performance in stochastic environments with significant sample efficiency improvement comparing to existing approaches and maintain competitive performance in deterministic setting. The theoretical analysis provides guarantees on the optimal policy invariance under the transition-aware reward shaping and highlights the relationship between performance difference and transition model's estimation error. Empirical evaluations on MuJoCo and Atari benchmark environments validate the effectiveness of our method. Future works will focus on extending to various generative structure for reward learning and exploring extensions of the framework to multi-agent and hierarchical reinforcement learning scenarios. Overall, our approach offers a promising direction for advancing model-based adversarial IRL, with the potential to scale to a broader range of real-world applications.

\acks{We thank a bunch of people.}

\bibliography{l4dc2026-sample}
\newpage


\appendix


\vspace{-0.2cm}
\section{Related Works}
\label{section::related_work}
\paragraph{Generative IRL.}
Margin optimization based IRL methods \citep{ng2000algorithms, abbeel2004apprenticeship, ratliff2006maximum} aim to learn reward functions that explain expert behavior better than other policies by a margin. Bayesian approaches were introduced with different prior assumptions on reward distributions, such as Boltzmann distributions \citep{ramachandran2007bayesian, choi2011map, chan2021scalable} or Gaussian Processes \citep{levine2011nonlinear}. To avoid biases from maximum margin methods, \citep{ziebart2008maximum,ziebart2010modeling} proposed a Lagrangian dual framework to cast the reward learning into a maximum likelihood problem with linear-weighted feature-based reward representation. \cite{wulfmeier2015maximum} extended the framework to nonlinear reward representations, and \cite{finn2016guided} combined importance sampling techniques to enable model-free estimation. Inspired by GANs, generative methods were introduced for policy and reward learning in IRL~\citep{ho2016generative, fu2017learning, torabi2018generative}. However, \textbf{these methods typically work with Maximum Entropy (ME) formulation yet suffer from sample inefficiency and stochasticity}. Although there have been efforts to combine generative methods with off-policy RL agents to improve sample efficiency \citep{kostrikov2018discriminator, blonde2019sample, blonde2022lipschitzness}, few extend it to the model-based setting which might further the improvement, and none of these approaches addresses the rewards learning in stochastic MDP.
\paragraph{MBIRL.}
Integrating IRL with MBRL has also shown success. For example, \cite{das2021model} and \cite{herman2016inverse} presented a gradient-based IRL approach using different policy optimization methods with dynamic models for linear-weighted features reward learning. In \cite{das2021model}, the dynamic model is used to pass forward/backward the gradient in order to update the IRL and policy optimization modules. Similarly, end-to-end differentiable adversarial IRL frameworks to various state spaces have also been explored \citep{baram2016model, baram2017end, sun2021adversarial, rafailov2021visual}, where dynamic model serves a similar role. Despite these advancements, existing methods rarely address the specific challenges posed by stochastic environments, which limit reward learning performance. 
\paragraph{Reward Shaping.}
Reward shaping~\citep{dorigo1994robot,randlov1998learning} is a technique that enhances the original reward signal by adding additional domain information, making it easier for the agent to learn optimal behavior. This can be defines as $\hat{R}=R+F$, where $F$ is the shaping function and $\hat{R}$ is the shaped reward function. Potential-based reward shaping (PBRS)~\citep{ng2000algorithms} builds the potential function on states, $F(s,a,s')=\phi(s')-\phi(s)$, while ensuring the policy invariance property, which refers to inducing the same optimal behavior under different rewards $R$ and $\hat{R}$. Nonetheless, there exists other variants on the inputs of the potential functions such as state-action~\citep{wiewiora2003principled}, state-time~\citep{devlin2012dynamic}, and value function~\citep{harutyunyan2015expressing} as potential function input. There are also some recent attempts of reward shaping without utilization of domain knowledge potential function to solve exploration under sparse rewards~\citep{hu2020learning,devidze2022exploration,gupta2022unpacking,skalse2023invariance}.


\newpage
\appendix
\section{Reward Shaping Soft Optimal Policy}
\label{proof::soft_optimal_policy}
\begin{theorem}
    Let $R$ and $\hat{R}$ be two reward functions. $R$ and $\hat{R}$ induce the same soft optimal policy under all transition dynamics $\mathcal{T}$ if $\hat{R}(s_t,a_t)=R(s_t,a_t)+\gamma\mathbb{E}_{\mathcal{T}}[\phi(s_{t+1})\vert s_t,a_t] - \phi(s_t)$ for some potential-shaping function $\phi:\mathcal{S}\rightarrow\mathbb{R}$.
\end{theorem}

\begin{proof}
According to Soft VI (~\cref{equation::soft_VI}), we can expand the representation of $Q^{soft}_{\hat{R}}(s_t,a_t)$ as follows.
    \begin{align*}
        Q^{soft}_{\hat{R}}(s_t,a_t) &= R(s_t,a_t)+\gamma\mathbb{E}_{\mathcal{T}}\left[\phi(s_{t+1})\vert s_t, a_t\right]-\phi(s_t)+\gamma\mathbb{E}_{\mathcal{T}}\left[V_{\hat{R}}^{soft}(s_{t+1})\vert s_t, a_t\right],\\
        Q^{soft}_{\hat{R}}(s_t,a_t) + \phi(s_t) &= R(s_t,a_t) + \gamma\mathbb{E}_{\mathcal{T}}\left[V_{\hat{R}}^{soft}(s_{t+1})+\phi(s_{t+1})\vert s_t, a_t\right],\\
        Q^{soft}_{\hat{R}}(s_t,a_t) + \phi(s_t) &= R(s_t,a_t) + \gamma\mathbb{E}_{\mathcal{T}}\left[\log\sum_{a\in\mathcal{A}}\exp\left(Q_{\hat{R}}^{soft}(s_{t+1},a)+\phi(s_{t+1})\right)\vert s_t, a_t\right].
    \end{align*}
    From above derivation, we can tell that $Q^{soft}_{\hat{R}}(s_t,a_t) + \phi(s_t)$ satisfy the soft bellmen update with original $R$. Thus, with simple induction, we can arrive that $Q_{R}^{soft}(s_t,a_t) = Q_{\hat{R}}^{soft}(s_t,a_t)+\phi(s_t)$.
    Then, we can derive the advantage function
    \begin{align*}
        A_{\hat{R}}^{soft}(s_t,a_t) &= Q_{\hat{R}}^{soft}(s,a) - V_{\hat{R}}^{soft}(s_t)\\
        &= Q_{\hat{R}}^{soft}(s_t,a_t)-\log\sum_{a\in\mathcal{A}}\exp\left(Q_{\hat{R}}^{soft}(s_t,a_t)\right)\\
        &= Q_{\hat{R}}^{soft}(s_t,a_t) + \phi(s_t) -\log\sum_{a\in\mathcal{A}}\exp\left(Q_{\hat{R}}^{soft}(s_t,a_t) + \phi(s_t)\right)\\
        &= Q_{R}^{soft}(s_t,a_t)-\log\sum_{a\in\mathcal{A}}\exp\left(Q_{R}^{soft}(s_t,a_t)\right)\\
        &= A_{R}^{soft}(s_t,a_t).
    \end{align*}
\end{proof}

\section{Adversarial Reward Learning}
\label{proof::gradient alignment}
\begin{proposition}
    Consider an undiscounted MDP. Suppose $f_\theta$ and $\pi$ at current iteration are the soft-optimal advantage function and policy for reward function $R_\theta$. Minimising the cross-entropy loss of the discriminator under generator $\pi$ is equivalent to maximising the log-likelihood under Maximum Causal Entropy IRL. 
\end{proposition}

\begin{proof}
\begin{align*}
    \mathcal{L}_{IRL}(\mathcal{D}_{exp},\theta) &= \mathbb{E}_{\mathcal{D}_{exp}}[\log p_\theta(\tau)]\\
    &= \mathbb{E}_{\mathcal{D}_{exp}}\left[\sum_{t=0}^{T-1}\log\pi(a_t|s_t)+\log\rho_0+\sum_{t=1}^T\log\mathcal{T}(s_{t+1}|s_t,a_t)\right]\\
    &= \mathbb{E}_{\mathcal{D}_{exp}}\left[\sum_{t=0}^{T-1}\left(Q^{soft}_\theta(s_t,a_t)-V^{soft}_\theta(s_t)\right)\right] + \text{constant}.
\end{align*}
Breaking down above equations with soft VI (\cref{equation::soft_VI}), we can arrive the following.
\begin{equation}
    \mathbb{E}_{\mathcal{D}_{exp}}\left[\sum_{t=0}^{T-1} R_\theta(s_t,a_t)\right]+\mathbb{E}_{\mathcal{D}_{exp}}\left[\sum_{t=0}^{T-2}\mathbb{E}_{\mathcal{T}}\left[V_\theta^{soft}(s_{t+1})|s_t,a_t\right]\right]-\mathbb{E}_{\mathcal{D}_{exp}}\left[\sum_{t=0}^{T-1} V_\theta^{soft}(s_t)\right].
\end{equation}
Next we will derive the gradient of the loss.
\begin{multline}
    \nabla_{\theta}\mathcal{L}(\mathcal{D}_{exp},\theta) = \underbrace{\nabla_{\theta}\mathbb{E}_{\mathcal{D}_{exp}}\left[\sum_{t=0}^{T-1} R_{\theta}(s_t,a_t)\right]}_{A} + \\\underbrace{\nabla_{\theta}\mathbb{E}_{\mathcal{D}_{exp}}\left[\sum_{t=0}^{T-2}\left(\mathbb{E}_{\mathcal{T}}\left[V_{\theta}^{soft}(s_{t+1})\vert s_t,a_t\right]\right) -V_{\theta}^{soft}(s_{t+1})\right]}_{B} - \underbrace{\nabla_{\theta}V_{\theta}^{soft}(s_0)}_C.
\end{multline}
Let's get explicit expression of each part.
\begin{align*}
    A &= \mathbb{E}_{\mathcal{D}_{exp}}\left[\sum_{t=0}^{T-1}\nabla_{\theta}R_{\theta}(s_t,a_t)\right]\\
    C &=\nabla_\theta\log\sum_{a_t\in\mathcal{A}}\exp{Q_\theta^{soft}(s_t,a_t)}\\
    &=\sum_{a_t\in\mathcal{A}}\pi(a_t\vert s_t)\nabla_\theta Q_{\theta}^{soft}(s_t,a_t)\\
    & = \mathbb{E}_{\pi}\left[\sum_{t=0}^{T-1}\nabla_\theta R_{\theta}(s_t,a_t)\right].
\end{align*}
In our case, the transition function $\mathcal{T}$ is estimated by an approximation function $\hat{\mathcal{T}}$, which is updated  with samples from $\mathcal{D}_{exp}$ and samples from off-policy buffer $\mathcal{D}_{env}$, thus we can drop the $\mathbb{E}_{\mathcal{T}}$ here. And $B$ term will cancel out, ending up to $0$. To summarize, the gradient of log MLE loss of MCE IRL is the following.
\begin{equation}
    \nabla_\theta\mathcal{L}_{IRL}(\mathcal{D}_{exp},\theta)=\mathbb{E}_{\mathcal{D}_{exp}}\left[\sum_{t=0}^{T-1}\nabla_{\theta}R_{\theta}(s_t,a_t)\right] -\mathbb{E}_{\pi}\left[\sum_{t=0}^{T-1}\nabla_\theta R_{\theta}(s_t,a_t)\right].
\end{equation}
Next, we will start to derive the gradient of cross-entropy discriminator training loss. Remember the discriminator loss is defined in Eq~\ref{equation::disc_loss}.
\begin{align*}
    \log D_\theta(s,a,\mathcal{T}) &= f_\theta(s,a,\mathcal{T})-\log(\exp\{f_\theta(s,a,\mathcal{T})\}+\pi(a|s)),\\
    \log(1-D_\theta(s,a,\mathcal{T})) &= \log\pi(a|s)-\log(\exp\{f_\theta(s,a,\mathcal{T})\}+\pi(a|s)).
\end{align*}
Then, the gradient of each term is as follow:
\begin{align*}
    \nabla_\theta\log D_\theta(s,a,\mathcal{T}) &= \nabla_\theta f_\theta(s,a,\mathcal{T})-\frac{\exp\{f_\theta(s,a,\mathcal{T})\}\nabla_\theta f_\theta(s,a,\mathcal{T})}{\exp\{f_\theta(s,a,\mathcal{T})\}+\pi(a|s)},\\
    \nabla_\theta\log(1-D_\theta(s,a,\mathcal{T})) &= -\frac{\exp\{f_\theta(s,a,\mathcal{T})\}\nabla_\theta f_\theta(s,a,\mathcal{T})}{\exp\{f_\theta(s,a,\mathcal{T})\}+\pi(a|s)}.
\end{align*}
Since $\pi$ is trained by using $f_\theta$ as shaped reward, from soft VI we can derive that $\pi^*_{f_\theta}(a|s)=\exp A^{soft}_{f_\theta}(s,a)$. By assumption, we assume that $f_\theta$ is the advantage function of $R_\theta$, $f_\theta(s,a)=A^{soft}_{R_\theta}(s,a)$. From \tref{thm::soft_optimal_policy}, we know that $A^{soft}_{R_\theta}(s,a)=A^{soft}_{f_\theta}(s,a)$, which also implies that $\pi^*_{f_\theta}=\pi^*_{R_\theta}$. Then, we can deduce the gradient of the loss of discriminator.
\begin{align*}
    -\nabla_\theta\mathcal{L}_{disc} &= \mathbb{E}_{\mathcal{D}_{exp}}\left[\nabla_\theta\log D_\theta(s,a,\mathcal{T})\right]+\mathbb{E}_{\pi}\left[\nabla_\theta\log(1-D_\theta(s,a,\mathcal{T}))\right]\\
    &= \mathbb{E}_{\mathcal{D}_{exp}}\left[\frac{1}{2}\nabla_\theta f_\theta(s,a,\mathcal{T})\right]-\mathbb{E}_{\pi}\left[\frac{1}{2}\nabla_\theta f_\theta(s,a,\mathcal{T})\right],\\
    -2\nabla_\theta\mathcal{L}_{disc} &= \mathbb{E}_{\mathcal{D}_{exp}}\left[\nabla _\theta f_\theta(s,a,\mathcal{T})\right] - \mathbb{E}_{\pi}\left[\nabla_\theta f_\theta(s,a,\mathcal{T})\right].
\end{align*}
\end{proof}

\section{Performance Gap Analysis}
\begin{lemma}[Implicit Feasible Reward Set~\citep{ng2000algorithms}]
Let $\mathfrak{B}=(\mathcal{M}',\pi^*)$ be an IRL problem. Then $R\in\mathcal{R}_\mathfrak{B}$ if and only if for all $(s,a)\in\mathcal{S}\times\mathcal{A}$ the following holds:
\begin{align*}
    Q_{\mathcal{M}'\cup R}^{\pi^*}(s,a)-V_{\mathcal{M}'\cup R}^{\pi^*}(s) &= 0\quad \textrm{if}\; \pi^*(a|s)>0,\\
    Q_{\mathcal{M}'\cup R}^{\pi^*}(s,a)-V_{\mathcal{M}'\cup R}^{\pi^*}(s) &\leq 0\quad \textrm{if}\; \pi^*(a|s)=0.
\end{align*}
\end{lemma}
Combined with the traditional Value Iteration of RL, we can write out the explicit form of the reward function $R$.
\begin{lemma}[Explicit Feasible Reward Function~\citep{metelli2021provably}]
\label{lemma:feasi_reward}
With the above lemma conditions, $R\in\mathcal{R}_\mathfrak{B}$ if and only if there exist $\xi\in\mathbb{R}_{\geq 0}^{\mathcal{S}\times\mathcal{A}}$ and value function $V\in\mathbb{R}^\mathcal{S}$ such that:
\begin{equation}\label{equation::explicit_reward}
    R(s,a) = V(s) - \gamma \sum_{s'\in\mathcal{S}} \mathcal{T}(s'|s,a)V(s') - \xi(s,a)\mathbb{I}\{\pi^*(a|s)=0\}.
\end{equation}
\end{lemma}
With ~\cref{equation::explicit_reward}, we can derive the following error bound between $R\in \mathcal{R}^E_{\mathfrak{B}}$ and $\hat{R}^E\in\mathcal{R}_{\hat{\mathfrak{B}}}$. 


\begin{theorem}[Reward Function Error Bound]
\label{proof:reward_error_bound}
    Let $\mathfrak{B}=(\mathcal{M}',\pi^*)$ and $\hat{\mathfrak{B}}=(\hat{\mathcal{M}'},\pi^*)$ be two IRL problems, then for any $R^E\in\mathcal{R}_{\mathfrak{B}}$ there is a corresponding $\hat{R}^E\in\mathcal{R}_{\hat{\mathfrak{B}}}$ such that
    \begin{equation}
        \Vert R^E-\hat{R}^E\Vert\leq\frac{\gamma}{1-\gamma}|\mathcal{S}|\epsilon_{\mathcal{T}}R_{\max}.
    \end{equation}
\end{theorem}
\begin{proof}
From ~\cref{lemma:feasi_reward}, we can derive the following representations of $R$ and $\hat{R}$ with the same set of $V$ and $\xi$:
\begin{align*}
    R^E(s,a) &= V(s) - \gamma \sum_{s' \in \mathcal{S}} \mathcal{T}(s'|s,a)V(s') - \xi(s,a)\mathbb{I}\{\pi^*(a|s) = 0\}, \\
    \hat{R}^E(s,a) &= V(s) - \gamma \sum_{s' \in \mathcal{S}} \hat{\mathcal{T}}(s'|s,a)V(s') - \xi(s,a)\mathbb{I}\{\pi^*(a|s) = 0\}.
\end{align*}

The difference between $R^E$ and $\hat{R}^E$ can be bounded as follows:
\begin{align*}
    \Vert R^E - \hat{R}^E \Vert &\leq \gamma \sum_{s' \in \mathcal{S}} D_{TV}\left(\mathcal{T}(s'|s,a)\vert\hat{\mathcal{T}}(s'|s,a)\right) \cdot \Vert V(s')\Vert.
\end{align*}

Given that the total variation distance between the two dynamics is bounded by $\epsilon_{\mathcal{T}}$, and the reward function is bounded by $R_{\max}$, together with the definition of the value function, we have $\Vert V \Vert_{\infty} \leq \frac{R_{\max}}{1 - \gamma}$. Substituting these bounds, we derive the following inequality:
\begin{equation*}
    \Vert R^E - \hat{R}^E\Vert \leq \frac{\gamma}{1 - \gamma} |\mathcal{S}| \epsilon_{\mathcal{T}} R_{\max}.
\end{equation*}
\end{proof}

Next, we will propagate this bound to the value functions of optimal policy regarding different reward functions $R^E$ and $\hat{R}^E$. From the traditional Value iteration, we can write out the value function. 
\begin{equation}
    V_{\mathcal{M}'\cup R^E}^\pi(s) =\sum_{a\in\mathcal{A}}\pi(a|s)\sum_{s'\in\mathcal{S}}\mathcal{T}(s'|s,a)\left[R^E(s,a)+\gamma V_{\mathcal{M}'\cup R^E}^\pi(s')\right].
\end{equation}

\begin{lemma}[Value Function Error under same policy and different rewards and MDP]
    $||V_{\mathcal{M}'\cup R}^\pi(s) - V_{\hat{\mathcal{M}}'\cup \hat{R}}^\pi(s)||$: the performance difference of the same policy in different MDPs.
    \begin{equation}
        ||V_{\mathcal{M}'\cup R^E}^\pi(s) - V_{\hat{\mathcal{M}}'\cup \hat{R}^E}^\pi(s)||\leq
\epsilon_\mathcal{T} \frac{1+\gamma}{(1-\gamma)^2}R_{\max} |\mathcal{S}|.
    \end{equation}
\end{lemma}

\begin{proof}

\begin{equation}
\begin{aligned}
&||V_{\mathcal{M}'\cup R^E}^\pi(s) - V_{\hat{\mathcal{M}}'\cup \hat{R}^E}^\pi(s)||\\
&\leq\sum_{a\in\mathcal{A}}\pi(a|s)\sum_{s'\in\mathcal{S}}||\mathcal{T}(s'|s,a)\left[R^E(s,a)+\gamma V_{\mathcal{M}'\cup R^E}^\pi(s')\right]-\hat{\mathcal{T}}(s'|s,a)\left[R^E(s,a)+\gamma V_{\mathcal{M}'\cup R^E}^\pi(s')\right]\\
&+\hat{\mathcal{T}}(s'|s,a)\left[R^E(s,a)+\gamma V_{\mathcal{M}'\cup R^E}^\pi(s')\right]-\hat{\mathcal{T}}(s'|s,a)\left[\hat{R}^E(s,a)+\gamma  V_{\hat{\mathcal{M}}'\cup \hat{R}^E}^\pi(s')\right]||\\
&\leq \sum_{a\in\mathcal{A}}\pi(a|s)\sum_{s'\in\mathcal{S}} (\epsilon_\mathcal{T} \frac{R_{\max}}{1-\gamma}+\hat{\mathcal{T}}(s'|s,a)\epsilon_{\mathcal{T}} (\frac{\gamma R_{\max}}{1-\gamma}|\mathcal{S}| + \gamma ||V_{\mathcal{M}'\cup R^E}^\pi(s') - V_{\hat{\mathcal{M}}'\cup \hat{R}^E}^\pi(s')||))
\\
&=\sum_{a\in\mathcal{A}}\pi(a|s)(\epsilon_\mathcal{T} \frac{R_{\max}}{1-\gamma}|\mathcal{S}|+\epsilon_{\mathcal{T}} \frac{\gamma R_{\max}}{1-\gamma}|\mathcal{S}|+ \gamma ||V_{\mathcal{M}'\cup R^E}^\pi(s') - V_{\hat{\mathcal{M}}'\cup \hat{R}^E}^\pi(s')||)\\
&\leq\epsilon_\mathcal{T} \frac{1+\gamma}{1-\gamma}R_{\max} |\mathcal{S}|
+ \gamma ||V_{\mathcal{M}'\cup R^E}^\pi(s') - V_{\hat{\mathcal{M}}'\cup \hat{R}^E}^\pi(s')||\\
&\leq\epsilon_\mathcal{T} \frac{1+\gamma}{(1-\gamma)^2}R_{\max} |\mathcal{S}|.
\end{aligned}
\end{equation}

\end{proof}

\begin{lemma}
\label{lemma:different_policy}
    Let $\Vert V_{\hat{\mathcal{M}'}\cup\hat{R}^E}^{\pi_{1}}(s) - V_{\hat{\mathcal{M}'}\cup\hat{R}^E}^{\pi_{2}}(s)\Vert$ denote the performance difference between different policies $\pi_1$ and $\pi_2$ in the same learned MDP \citep{viano2021robust, zhang2020learning}. The following inequality holds:
    \begin{equation*}
        \Vert V_{\hat{\mathcal{M}'}\cup\hat{R}^E}^{\pi_{1}}(s) - V_{\hat{\mathcal{M}'}\cup\hat{R}^E}^{\pi_{2}}(s)\Vert \leq \frac{\gamma}{(1 - \gamma)^2} \epsilon_{\mathcal{T}} R_{\max}.
    \end{equation*}
\end{lemma}


\begin{theorem}[Performance Difference Bound]
\label{proof:value_error_bound}
    The performance difference between the optimal policies ($\pi^*$ and $\hat{\pi}^*$) in corresponding MDPs ($\mathcal{M}'\cup R^E$ and $\hat{\mathcal{M}}'\cup \hat{R}^E$) can be bounded as follows:
    \begin{equation}
        \Vert V_{\mathcal{M}'\cup R^E}^{\pi^*} - V_{\hat{\mathcal{M}}'\cup \hat{R}^E}^{\hat{\pi}^*} \Vert\leq
\epsilon_\mathcal{T}\left[ \frac{\gamma}{(1-\gamma)^2} R_{\max}+\frac{1+\gamma}{(1-\gamma)^2}R_{\max} |\mathcal{S}|\right].
    \end{equation}
\end{theorem}
\begin{proof}

\begin{align*}
    &||V_{\mathcal{M}'\cup R^E}^{\pi^*}(s) - V_{\hat{\mathcal{M}}'\cup \hat{R}^E}^{\hat{\pi}^*}(s)||\\
    &\leq||V_{\hat{\mathcal{M}}'\cup \hat{R}^E}^{\pi^*}(s) - V_{\hat{\mathcal{M}'}\cup \hat{R}^E}^{\hat{\pi}^*}(s)|| + ||V_{\mathcal{M}'\cup R^E}^{\hat{\pi}^*}(s) - V_{\hat{\mathcal{M}}'\cup \hat{R}^E}^{\hat{\pi}^*}(s)|| \\
    &=\epsilon_\mathcal{T} \frac{\gamma}{(1-\gamma)^2} R_{\max}+ \epsilon_\mathcal{T} \frac{1+\gamma}{(1-\gamma)^2}R_{\max} |\mathcal{S}|\\
    &=\epsilon_\mathcal{T}\left[\frac{\gamma}{(1-\gamma)^2}R_{\max}+\frac{1+\gamma}{(1-\gamma)^2}R_{\max}\vert\mathcal{S}\vert\right].
\end{align*}
\end{proof}

\newpage
\section{Ablation Studies}
\label{section:ablation_study}
\subsection{Robustness to stochasticity}
\label{section:ablation_stochastic}
In this study, we examine the robustness of our method across varying levels of stochasticity in the environment. Following the same setup as in our main experiments, we introduce an unknown Gaussian noise with different standard deviations in \texttt{InvertedPendulum-v4} to simulate increased stochasticity. As shown in \tableref{abalation_stochastic_table} and \cref{Fig:abalation_stoch}, our method consistently recovers expert-level performance despite the presence of stochastic disturbances. However, as the level of stochasticity increases, we observe that training stability decreases, as reflected in the increased variance in \cref{Fig:abalation_stoch}.

\begin{figure}[h]
    \centering
    \begin{minipage}{0.5\textwidth}
        \centering
        \includegraphics[width=0.95\linewidth]{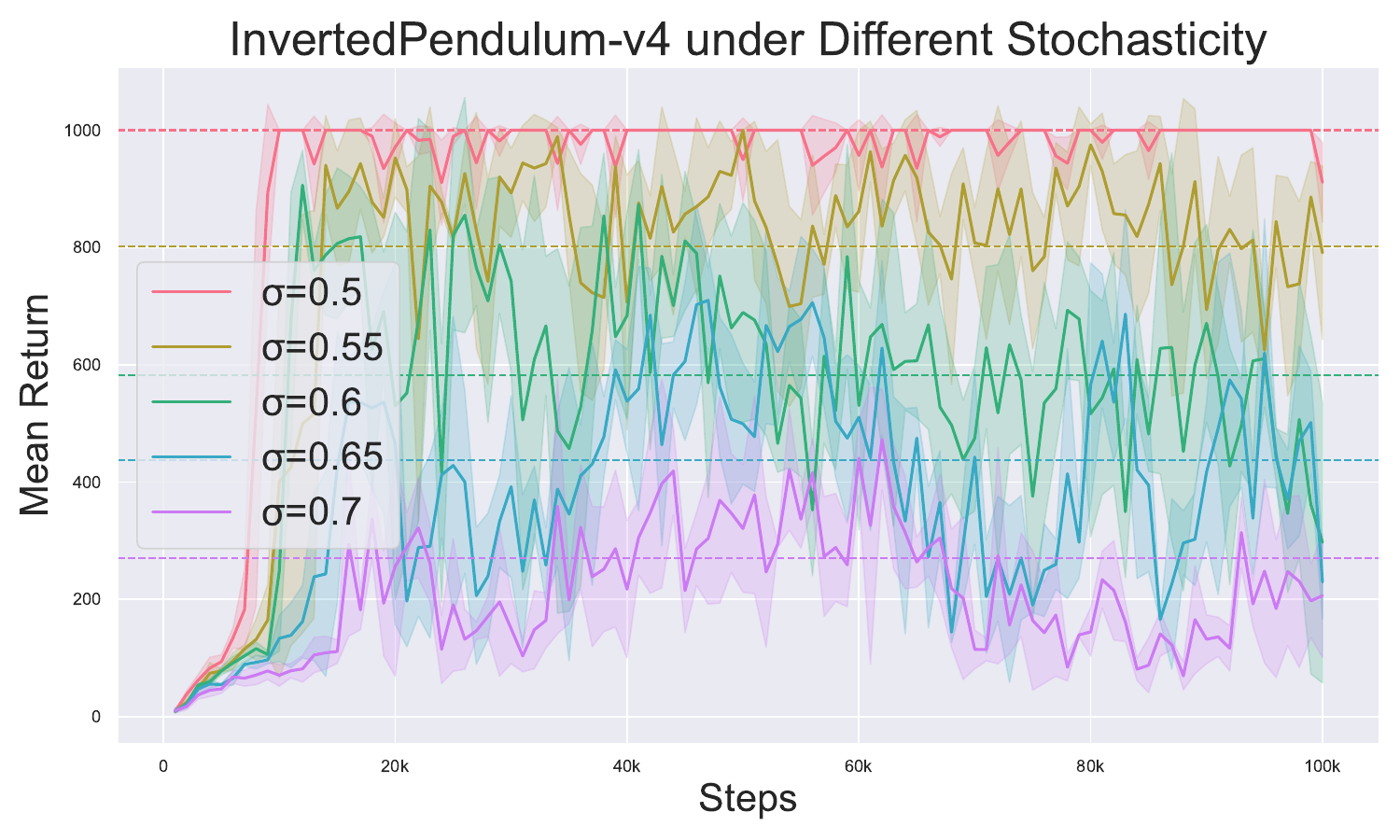}
        \captionof{figure}{Training return diagram averaging across three seeds for different numbers of expert trajectories in \texttt{InvertedPendulum-v4}.}
        \label{Fig:abalation_stoch}
    \end{minipage}
    \hfill
    \begin{minipage}{0.45\textwidth}
    
        \centering
        \scalebox{0.85}{
        \begin{tabular}{c|cc}
            \hline
            Std & Expert & Ours \\
            \hline
            0.5 & $1000.0_{\pm 0.0}$ & $1000.0_{\pm 0.0}$ \\
            0.55 & $802.4_{\pm 305.8}$ & $1000.0_{\pm 0.0}$ \\
            0.6 & $582.1_{\pm 360.5}$ & $906.3_{\pm 59.1 }$ \\
            0.65 & $438.2_{\pm 322.3}$ & $709.9_{\pm 80.6 }$ \\
            0.7 & $270.7_{\pm 236.0}$ & $472.7_{\pm 85.7}$ \\
            \hline
        \end{tabular}}
        \captionof{table}{Best performance of expert and our method in \texttt{InvertedPendulum-v4} environments with different Gaussian noises (standard deviations ranging from $0.5$ to $0.7$) for stochasticity under provided 100 expert trajectories.}\label{abalation_stochastic_table}
    \end{minipage}
\end{figure}

\subsection{Model estimation error and reward learning}
\label{section:ablation_model}

In this study, we empirically evaluate the effect of dynamic model learning errors on our method's performance, extending the theoretical analysis presented in \cref{section::performance_gap}. To isolate the impact of model errors specifically on reward learning, we use SAC on real trajectories for policy optimization, thereby removing any influence of model errors on trajectory generation that would typically affect model-based policy optimization. To quantify the relationship between model errors and performance, we standardize the ensemble model architecture as 2-layer MLPs with varying hidden layer dimensions from $8$ to $256$ to adjust model capacity. Our experiments are conducted in \texttt{HalfCheetah-v4} with random, policy-unknown Gaussian noise (mean $0$ and standard deviation $0.5$), as described in \cref{section::experiment}. 
From \cref{Fig:abalation_model} and \cref{Fig:abalation_model_error}, we observe the general trend that as modeling error decrease together with increasing capacity of the model structure, performances also increases, which is obvious when hidden dimension bumps up from $8$ to $16$ and $16$ to $32$. As transition model error narrows down, the performance improvement also becomes less obvious.    

\begin{figure}[h]
    \centering
    \begin{minipage}{0.45\textwidth}
        \centering
        \includegraphics[width=0.95\linewidth]{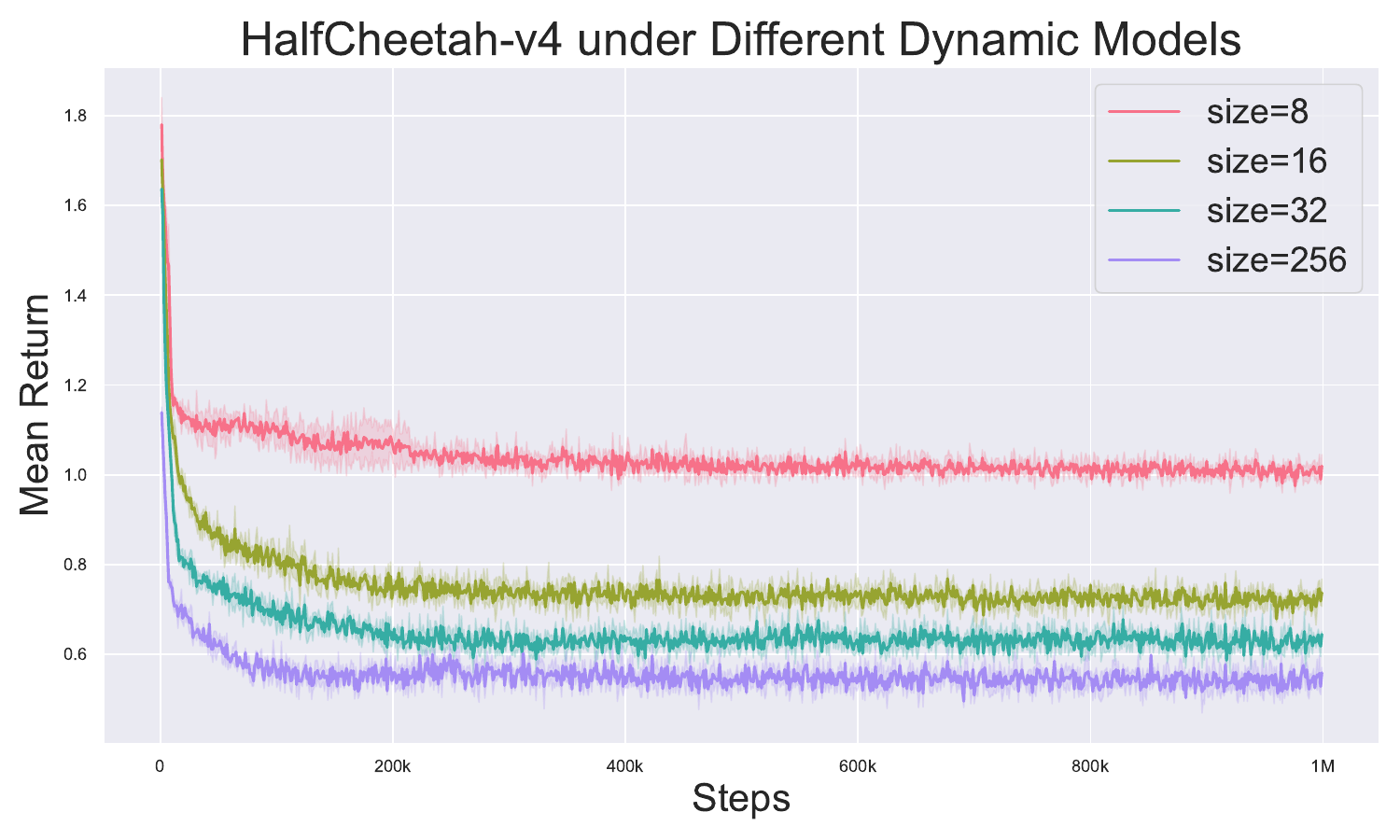}
        \captionof{figure}{Transition model learning error diagram averaging across three seeds for 10 expert trajectories in \texttt{HalfCheetah-v4}.}
        \label{Fig:abalation_model_error}
    \end{minipage}
    \hfill
    \begin{minipage}{0.5\textwidth}
        \centering
        \includegraphics[width=0.95\linewidth]{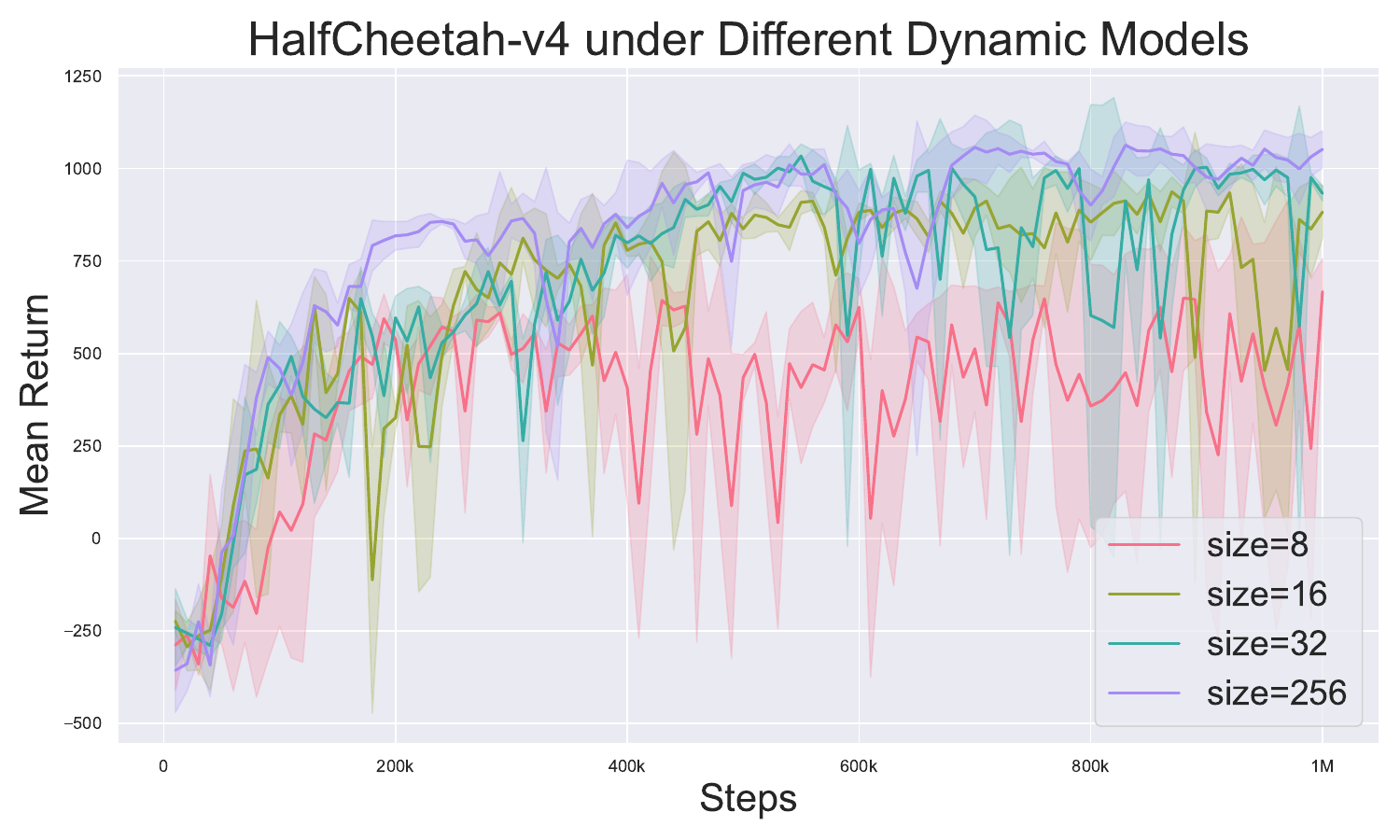}
        \captionof{figure}{Training return diagram averaging across three seeds for 10 expert trajectories in \texttt{HalfCheetah-v4}.}
        \label{Fig:abalation_model}
    \end{minipage}
\end{figure}

\subsection{Does model-based trajectories generation help?}
\label{section:ablation_sample_efficiency}
In this study, we empirically investigate the effectiveness of model-based policy optimization on our transition-aware reward shaping IRL framework. We compare three off-policy approaches namely DAC \citep{kostrikov2018discriminator}, transition-aware reward shaping with pure SAC for policy optimization (labeled as $mbirl\_sac$), and our original transition-aware reward shaping with model-based technique for policy optimization. Noted that synthetic data is also not used in reward learning in $mbirl\_sac$ approach. We conduct the experiment in stochastic \texttt{Hopper-v4} with 1000 expert trajectories. From \cref{Fig:abalation_policy} and \tableref{abalation_policy_table}, we can tell that both methods using transition-aware reward shaping have much better performance and sample efficiency compared to DAC. In terms of performance, both methods perform at a similar level. However, as the synthetic trajectories generation boosts the training process, our model-based method has better sample efficiency than the pure SAC-based method. 
\begin{figure}[h]
    \centering
    \begin{minipage}{0.5\textwidth}
        \centering
        \includegraphics[width=0.95\linewidth]{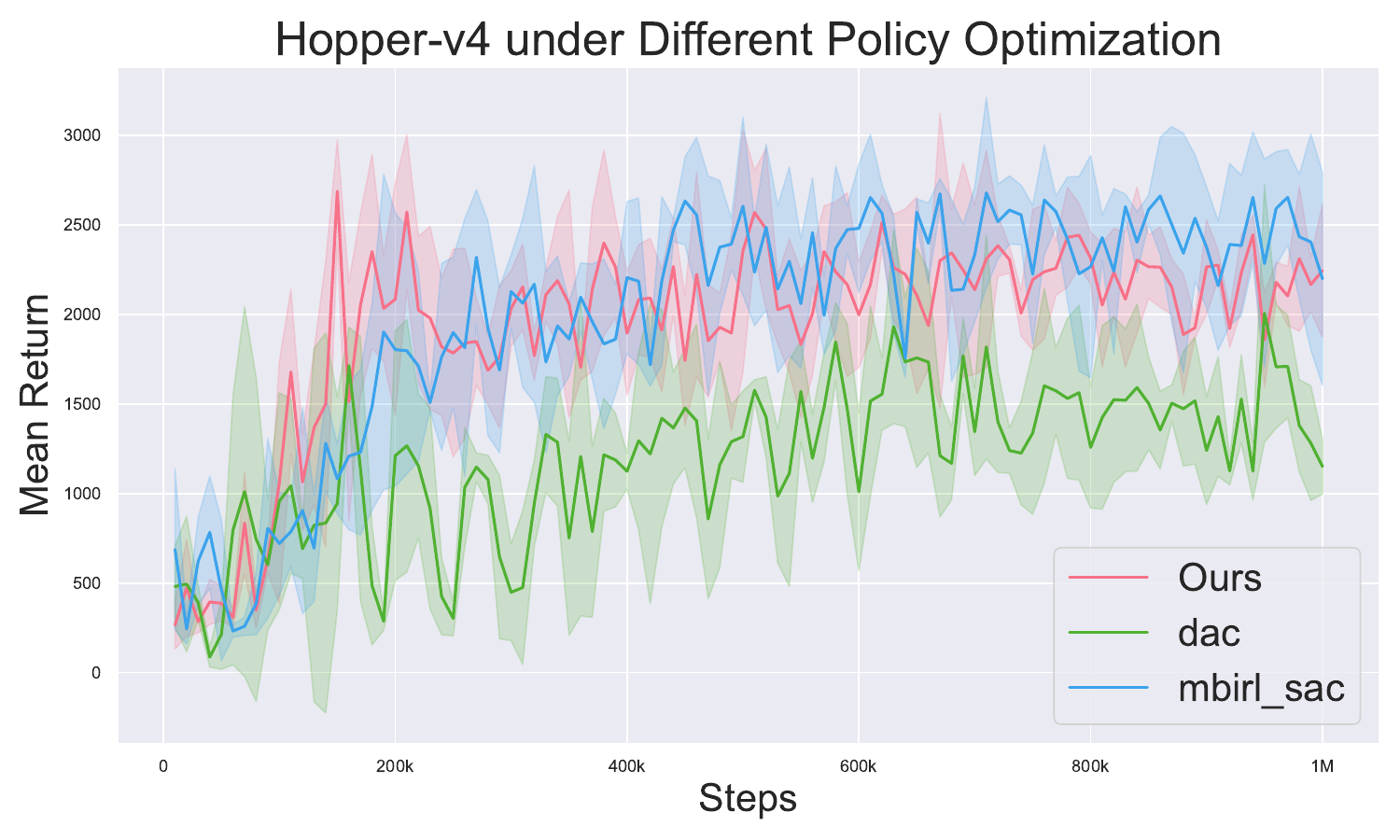}
        \captionof{figure}{Performance diagram averaging across three seeds for different algorithms in \texttt{Hopper-v4} with 1000 expert trajectories provided. DAC is in green color; $mbirl\_sac$ is in blue; Our method is in red.}
        \label{Fig:abalation_policy}
    \end{minipage}
    \hfill
    \begin{minipage}{0.45\textwidth}
        \centering
        \scalebox{0.95}{
        \begin{tabular}{c|c}
            \hline
            Method & Performance \\
            \hline   
            DAC & $2007.1_{\pm 719.7 }$ \\
            $mbirl\_sac$ & $2694.5_{\pm 77.5}$ \\
            Ours & $2798.8_{\pm 82.9 }$ \\
            \hline
        \end{tabular}}

        \captionof{table}{Best performance of three methods in stochastic \texttt{Hopper-v4} environment with under provided 1000 expert trajectories.}
            \label{abalation_policy_table}

    \end{minipage}
\end{figure}

\newpage
\section{Algorithmic Framework}

\begin{algorithm*}[h]
\caption{Model-Enhanced IRL (ME-IRL)}
\label{alg:meairl}
\begin{algorithmic}[1]
    \STATE \textbf{Input:} Expert buffer $\mathcal{D}_{exp}$ (collected expert trajectories), total training steps $N$, starting step for training policy $\text{STARTING\_STEP}$, horizon length $H$ for synthetic trajectory generation.
    \STATE \textbf{Output:} Learned policy $\pi$ optimized to mimic expert behavior.
    \STATE Initialize policy $\pi$ (random or pre-trained), discriminator $D_{\theta}$, buffers $\mathcal{D}_{env}$ (real environment samples), $\mathcal{D}_{gen}$ (synthetic samples), and transition model $\hat{\mathcal{T}}$ (random initialization). 
    
    \FOR{step $t$ in \{1, \dots, N\}}
        \STATE Interact with the environment using policy $\pi$ to collect new state-action pairs.
        \STATE Add the collected state-action pairs $(s_t, a_t, s_{t+1})$ to the real environment buffer $\mathcal{D}_{env}$.
        
        \IF{$t < \text{STARTING\_STEP}$}
            \STATE \textbf{Pre-train transition model:} Train $\hat{\mathcal{T}}$ on batches of samples from $\mathcal{D}_{env}$ using maximum likelihood estimation (MLE) loss to stabilize the transition model early.
        \ELSE
             \STATE \textbf{Step 1: Update Transition Model and Generate Synthetic Data}
            \STATE Update the transition model $\hat{\mathcal{T}}$ using maximum likelihood loss (MLE) on $\mathcal{D}_{env}$.
            \STATE Use $\hat{\mathcal{T}}$ to generate $H$-step synthetic trajectories and store them in the synthetic buffer $\mathcal{D}_{gen}$.
            \STATE \textbf{Step 2: Update Discriminator}
            \STATE Sample mini-batches of state-action pairs from $\mathcal{D}_{exp}$ (expert buffer) and $\mathcal{D}_{env}$ (environment buffer) with varying ratio.
            \STATE Train the discriminator $D_{\theta}$ using cross-entropy loss to classify expert data (from $\mathcal{D}_{exp}$) versus policy-generated data (from $\mathcal{D}_{env}$) as described in ~\cref{equation::disc_loss}.
            \STATE \textbf{Step 3: Policy Optimization with Mixed Data}
            \STATE Sample state-action batches from $\mathcal{D}_{env}$ and $\mathcal{D}_{gen}$ with a varying ratio that increases the use of $\mathcal{D}_{gen}$ as $\hat{\mathcal{T}}$ becomes more accurate.
            \STATE Compute the reward $\hat{R}_{\theta}$ for the sampled state-action pairs using the discriminator.
            \STATE Update the policy $\pi$ using Soft Actor-Critic (SAC) with the computed reward $\hat{R}_{\theta}$ as the optimization objective.
        \ENDIF
    \ENDFOR
    \STATE \textbf{Return:} Optimized policy $\pi$.
\end{algorithmic}
\end{algorithm*}

\newpage  
\section{Implementation Details}
\label{sec::implementation}
For our framework, we use two identical 2-layer Multi-Layer Perceptrons (MLPs) with 100 hidden units and ReLU activations for both the reward function \( R \) and the shaping potential function \( \phi \). To initialize the replay buffer for both \textbf{DAC} and ours, we collect 1,000 steps samples in \texttt{InvertedPendulum-v4} and \texttt{InvertedDoublePendulum-v4}, and 10,000 steps samples in \texttt{Hopper-v4}, \texttt{HalfCheetah-v4}, and \texttt{Walker2d-v4} with initial policy. During this pre-training phase, we also update the transition model at each step to mitigate divergence might happen at the beginning of the training. Additionally, the transition model is only trained using samples from real environment buffer $\mathcal{D}_{env}$ in policy optimization section before actor and critics updates during training phase. As discussed in \cref{section::algorithm}, the size of the synthetic data buffer \( \mathcal{D}_{\text{gen}} \) and the ratio of samples drawn from it increase as the model accuracy improves. Both parameters increase linearly with training steps, up to a maximum synthetic-to-real data ratio of 0.5 per training step and a maximum buffer size of \( 1 \) million samples in \( \mathcal{D}_{\text{gen}} \).
For consistency in comparisons, we used similar network structures and hyper-parameters for \textbf{AIRL, GAIL, and DAC} baselines, which we reference the implementations from \cite{arulkumaran2024pragmatic} and \cite{gleave2022imitation}. Detailed hyper-parameters for these networks are provided in the table below. For on-policy baselines \textbf{AIRL} and \textbf{GAIL}, the rollout length is set to 1,000 for \texttt{InvertedPendulum-v4} and \texttt{InvertedDoublePendulum-v4}, and 5,000 for \texttt{Hopper-v4}, \texttt{Walker2d-v4}, and \texttt{HalfCheetah-v4}. For the SAC and PPO policy optimization components, we reference implementations from the \texttt{CleanRL} repository \citep{huang2022cleanrl}. Our implementation of Dreamer-v2 and RSSM are based on SheepRL \citep{EclecticSheep_and_Angioni_SheepRL_2023}.

\begin{table}[ht]
\centering
\caption{Hyper-parameters table. \label{appendix:parameter}}
\begin{tabular}{|cc|}
\hline
Hyper-parameter                    & Value          \\ \hline
Seeds                              & 0, 5, 10       \\
Buffer Size                        & 1M              \\
Batch Size                         & 128            \\
Max Grad Norm                      & 10             \\
Starting Steps                     & 1,000/10,000   \\
Global Timesteps                   & 100k/1M        \\
Discount Factor                    & 0.99           \\ \hline
\multicolumn{2}{|c|}{Model-based Policy Optimization}              \\ \hline
Learning Rate for Actor            & 3e-4           \\
Learning Rate for Critic           & 3e-4           \\
Learning Rate for Model            & 3e-4           \\
Network Layers                     & 3              \\
Policy Network Neurons             & [64, 64]       \\
Critic Network Neurons             & [128, 128]     \\
Model Network Neurons              & [256, 256]     \\
Ensemble Numbers                   & 7              \\
Elites Numbers                     & 5              \\
Activation                         & Tanh(Policy)/ReLU\\
Optimizer                          & Adam           \\
Initial Entropy                    & $-|\mathcal{A}|$\\
Learning Rate for Entropy          & 3e-4           \\
Train Frequency for Actor          & 1              \\
Train Frequency for Critic         & 1              \\
Train Frequency for Model          & 1              \\
Synthetic and Real Data Mix Coef   & 0.5            \\
Horizon($H$)                            & 4         \\ \hline
\multicolumn{2}{|c|}{Adversarial Discriminator}  \\ \hline
Learning Rate                      & 3e-4           \\
$R$ Network Neurons                & [100, 100]     \\
$\phi$ Network Neurons             & [100, 100]     \\
Optimizer                          & Adam           \\ 
Loss                               & Binary Cross-Entropy \\ \hline
\end{tabular}
\end{table}

\newpage
\section{MuJuCo Graphical Results}
Below are the testing return diagrams from stochastic MuJoCo Environments. Since AIRL and GAIL use distinct environmental training steps from DAC and our method, ~\cref{fig:comparison_sample_efficiency} provides a clear comparison under \texttt{10M} landscape while the rest of the graphs show the sample efficiencies for all algorithms under \texttt{1M} landscape.  
\begin{figure*}[h]
    \centering
    \includegraphics[width=0.85\linewidth]{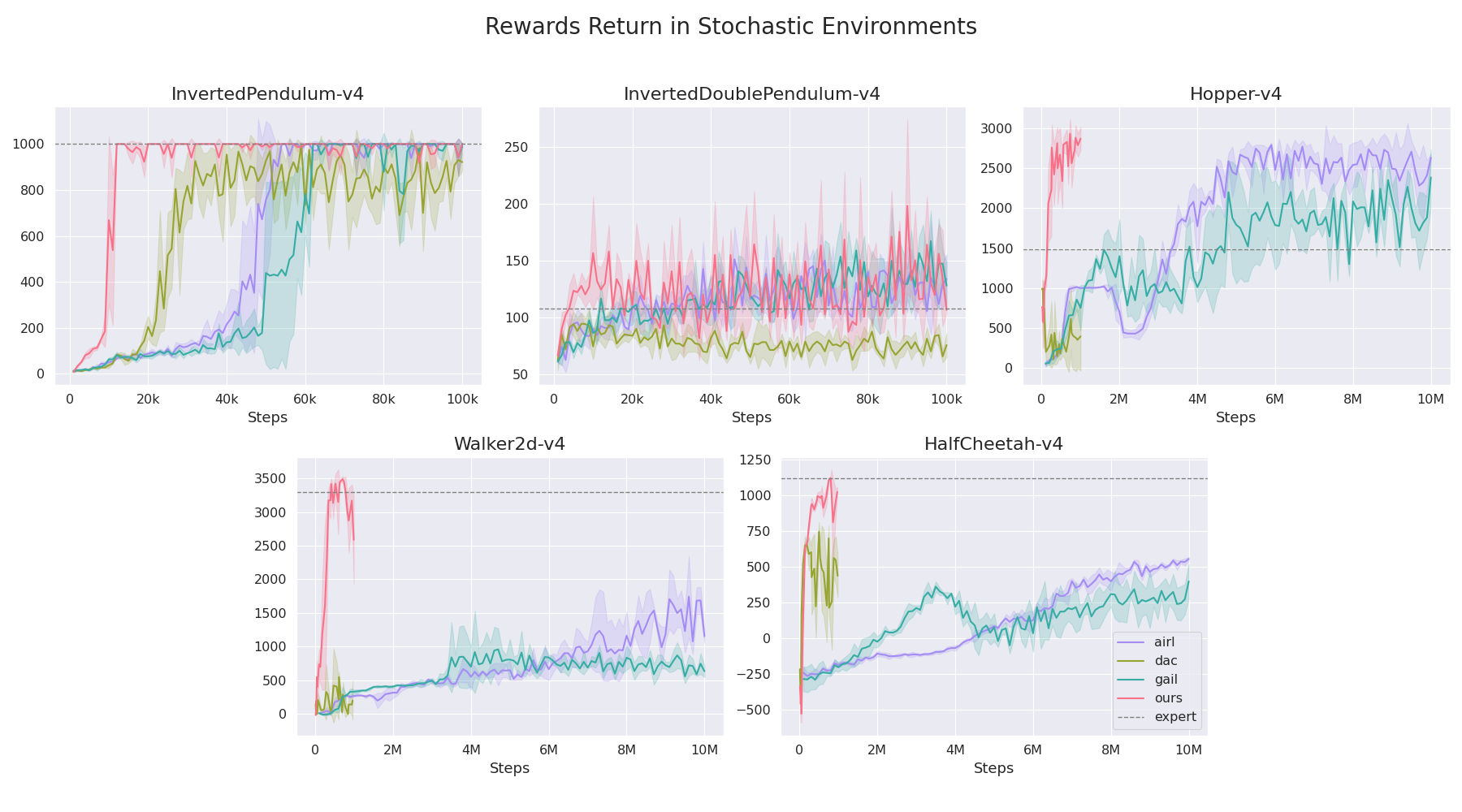}
    \caption{Training curves of all 4 methods in 5 different \textbf{stochastic} environments with 100 expert trajectories. For better comparison in sample efficiency, graph is presented under \texttt{10M} landscape.}
    \label{fig:comparison_sample_efficiency}
\end{figure*}
\label{deter_performance}

\newpage
\label{appendix:learning_curves}
\begin{figure}[h!]
    \centering
    \includegraphics[width=0.8\linewidth]{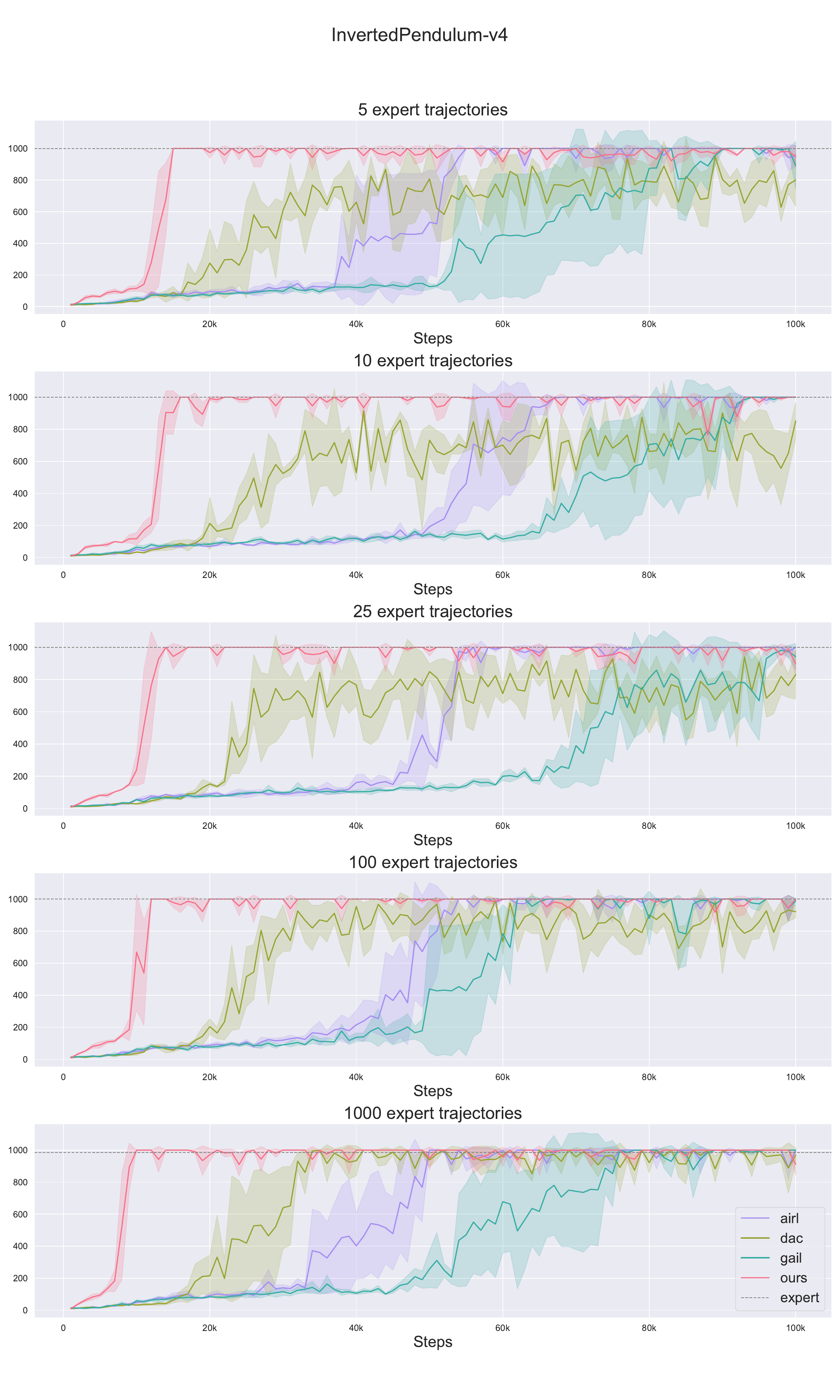}
    \caption{Training return diagram averaging across three seeds for different numbers of expert trajectories in Stochastic \texttt{InvertedPendulum-v4}.}
    \label{Fig:IP}
\end{figure}

\begin{figure}[h!]
    \centering
    \includegraphics[width=0.75\linewidth]{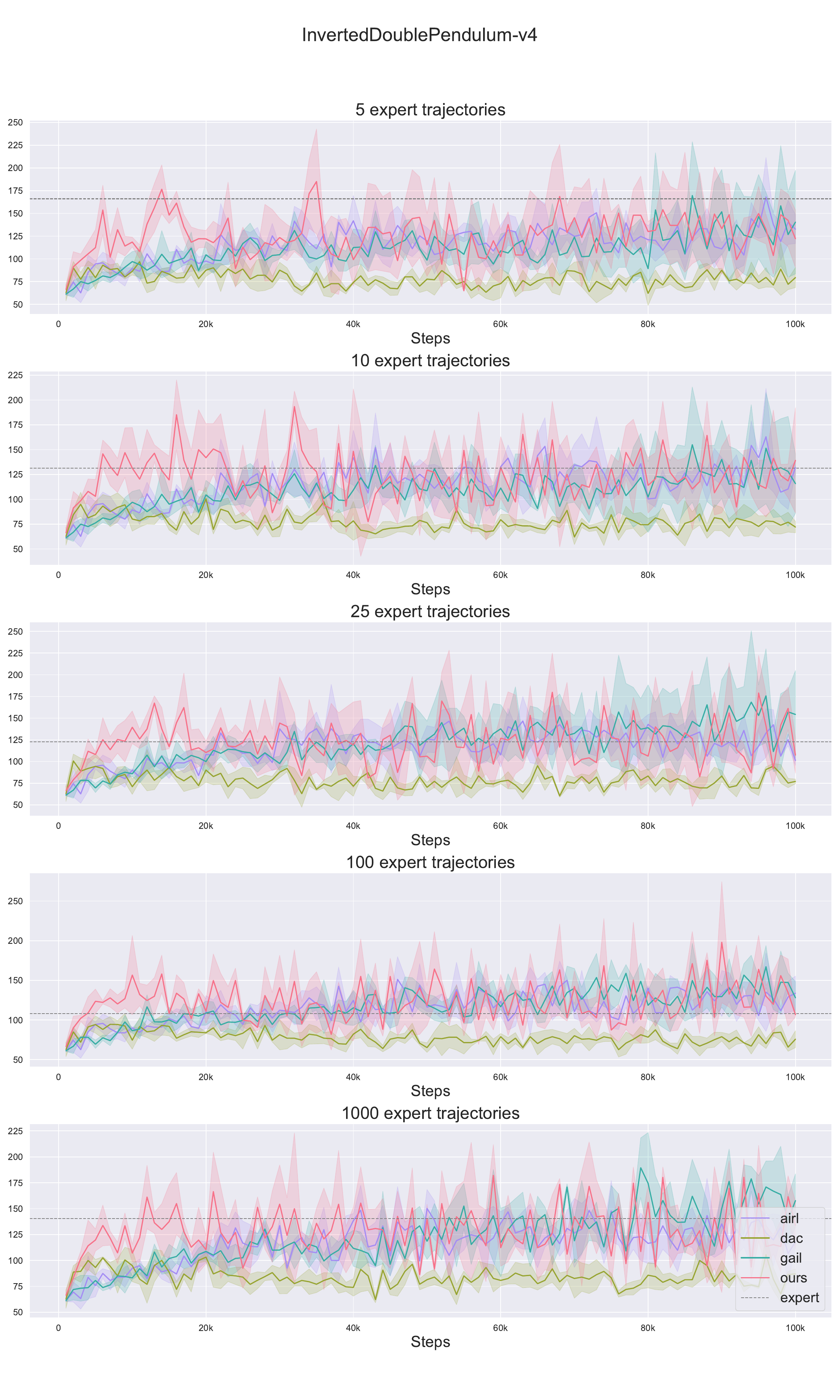}
    \caption{Training return diagram averaging across three seeds for different numbers of expert trajectories in Stochastic \texttt{InvertedDoublePendulum-v4}.}
    \label{Fig:IDP}
\end{figure}
\begin{figure}[h!]
    \centering
    \includegraphics[width=0.75\linewidth]{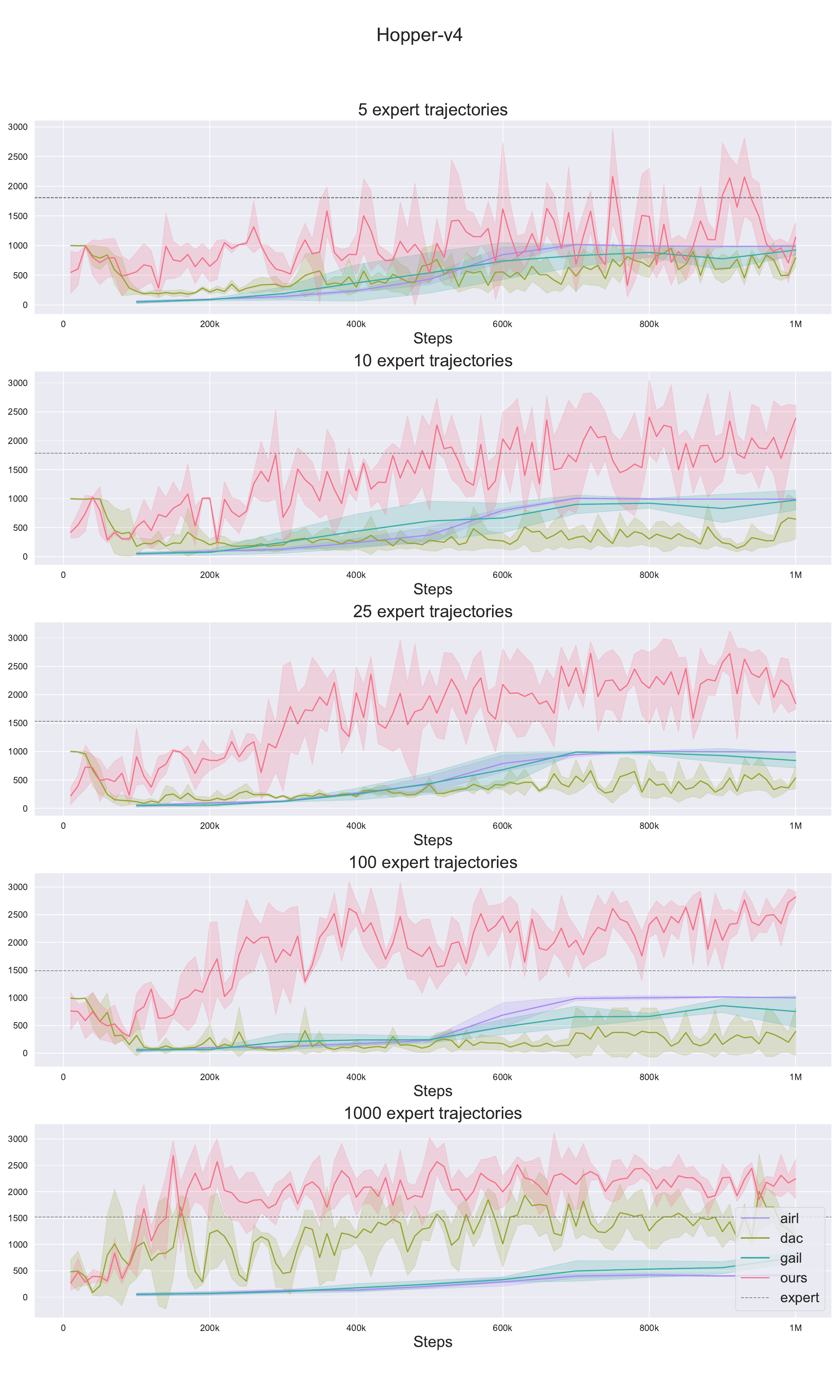}
    \caption{Training return diagram averaging across three seeds for different numbers of expert trajectories in Stochastic \texttt{Hopper-v4}.}
    \label{Fig:Hopper}
\end{figure}
\begin{figure}[h!]
    \centering
    \includegraphics[width=0.75\linewidth]{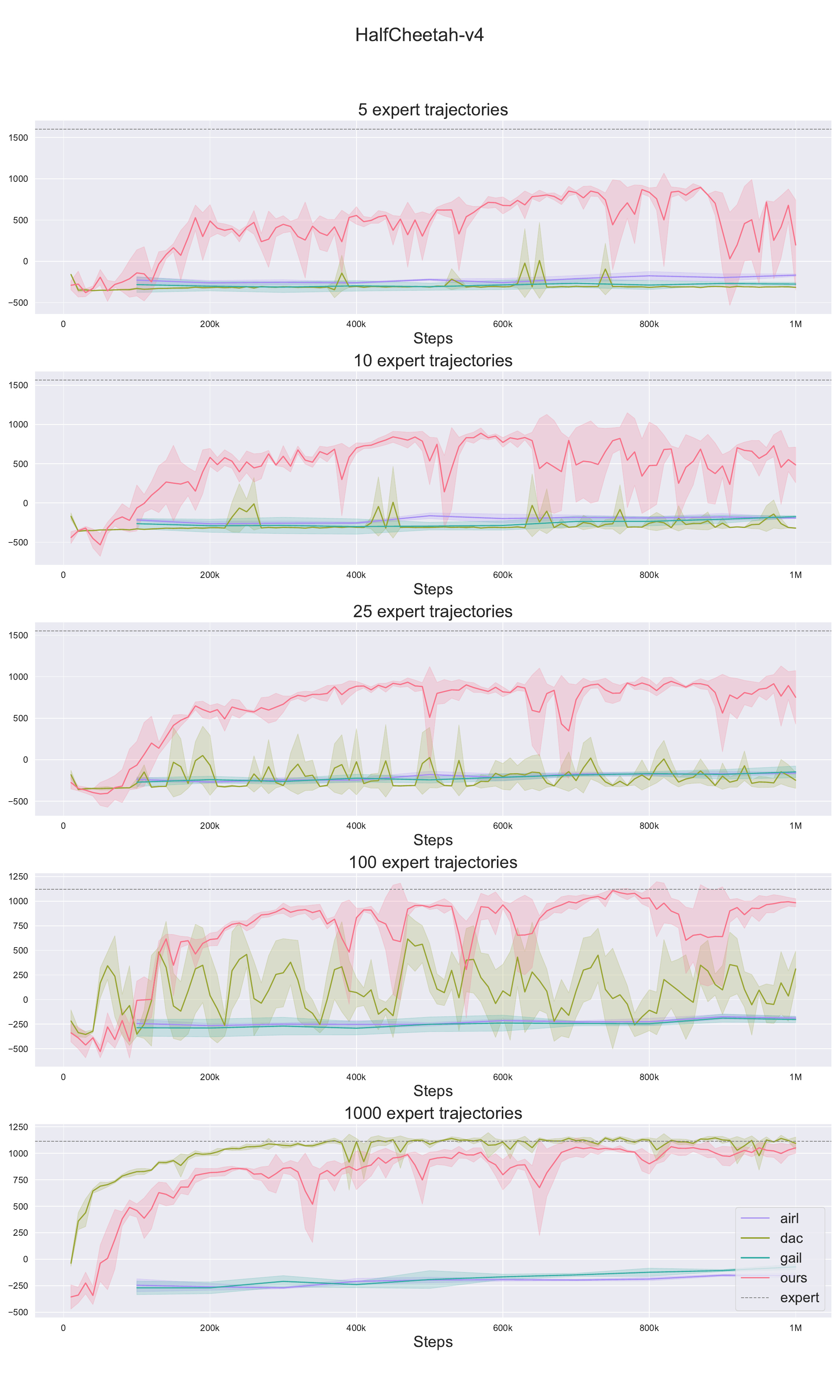}
    \caption{Training return diagram  averaging across three seeds for different numbers of expert trajectories in Stochastic \texttt{HalfCheetah-v4}.}
    \label{Fig:HalfCheetah}
\end{figure}
\begin{figure}[h!]
    \centering
    \includegraphics[width=0.75\linewidth]{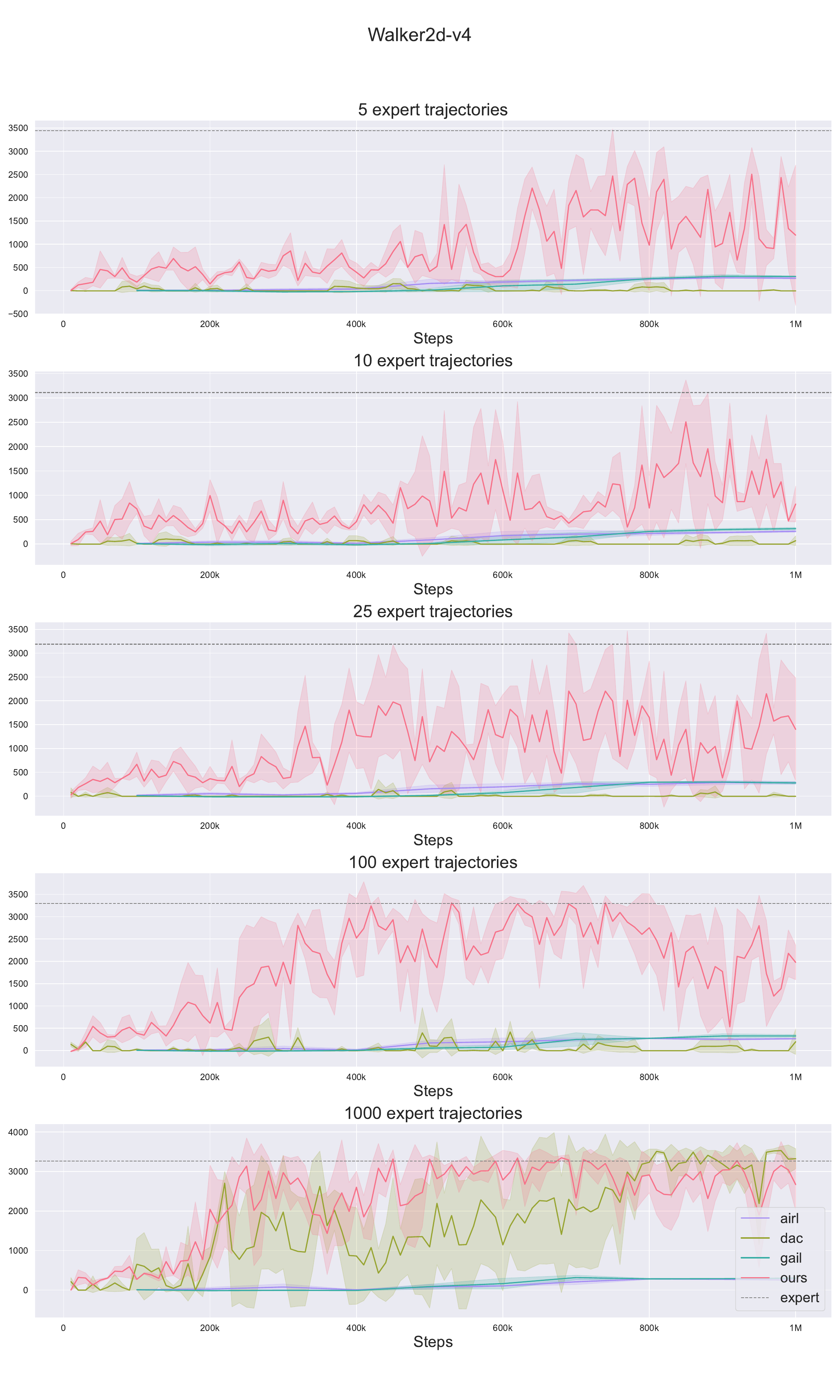}
    \caption{Training return diagram  averaging across three seeds for different numbers of expert trajectories in Stochsatic \texttt{Walker2d-v4}.}
    \label{Fig:Walker2d}
\end{figure}

\newpage
\section{Supplementary Material}
\subsection{Codespace}
To help make our results reproducible, we provide our implementation at the following anonymous link: 
\url{https://anonymous.4open.science/r/MBIRL-4C2F/README.md}

\begin{wraptable}[9]{r}{0.42\textwidth}  
  \vspace{-\baselineskip}             
  \centering
  \small                              
  \begin{tabular}{@{}lc@{}}           
    \toprule
    Method & Retern \\ \midrule
    Expert        & $22300.0_{\pm 3662.0}$ \\
    CNN-AIRL & $1200.0_{\pm 748.3}$ \\
    Ours       & $14400.0_{\pm 5607.1}$ \\ \bottomrule
  \end{tabular}
  \caption{Mean return over 10 evaluation episodes after \texttt{400k} steps on \texttt{BattleZone-v5}.}
  \label{tab:BattleZone_table}
\end{wraptable}
\subsection{Atari Experiment Details}
In this section, we explain the implementation details for the Atari experiments. We follow SheepRL \citep{EclecticSheep_and_Angioni_SheepRL_2023} for our implementation for the expert policy (Dreamer-v2 \citep{hafner2020mastering}) and the baseline IRL comparison (CNN-AIRL\citep{tucker2018inverse}). As for our own algorithm, we replace the ensemble dynamic model with RSSM\cite{hafner2019dream} to capture the high-dimensional inputs from Atari environments. To address the discrete action space, we use categorical distributions as actor networks for all algorithms and change the policy update logic accordingly. Due to computation resource constraint, we test our algorithms on two distinct environments - \texttt{SpaceInvaders-v5} and \texttt{BattleZone-v5} instead of the full set of Atari2600 environments. All algorithms are trained for \texttt{400k} steps. We report the mean return for \texttt{SpaceInvaders-v5} in \cref{tab:atari_table} and \texttt{BattleZone-v5} in \cref{tab:BattleZone_table}. The hyper-parameters of the RSSM can be found in \cref{appendix:atari-world-model}.

\begin{table}[ht]
\centering
\caption{RSSM hyper-parameters for Atari experiments.\label{appendix:atari-world-model}}
\begin{tabular}{|cc|}
\hline
\multicolumn{2}{|c|}{Policy Optimization for discrete actor} \\ \hline
Discrete Latent Size              & 32                                                    \\
Stochastic Latent Size            & 32                                                    \\
KL Balancing $\alpha$             & 0.8                                                   \\
Free Nats                         & 1.0 (averaged)                                        \\
KL Regulariser Scale              & 1.0                                                   \\ \hline
\multicolumn{2}{|c|}{Encoder / Observation Model} \\ \hline
CNN Channel Multiplier            & 48                                                    \\
MLP Layers                        & 4                                                     \\
Dense Units                       & 400                                                   \\
Activation                        & ELU / ELU                                             \\
Layer Norm                        & False                                                 \\ \hline
\multicolumn{2}{|c|}{Recurrent Model}          \\ \hline
Recurrent State Size              & 600                                                   \\
Dense Units                       & 400                                                   \\
Activation                        & ELU                                                   \\
Layer Norm                        & True                                                  \\ \hline
\multicolumn{2}{|c|}{Transition \& Representation Models} \\ \hline
Hidden Size                       & 600                                                   \\
Activation                        & ELU                                                   \\
Layer Norm                        & False                                                 \\ \hline
\multicolumn{2}{|c|}{Discount Model} \\ \hline
Dense Units                       & 400                                                   \\
Activation                        & ELU / ELU                                             \\
Layer Norm                        & False                                                 \\
\hline
\end{tabular}
\end{table}
\newpage

\end{document}